\documentclass{article}

\usepackage[preprint]{neurips_2026}

\usepackage{placeins}
\usepackage[utf8]{inputenc} % allow utf-8 input
\usepackage[T1]{fontenc}    % use 8-bit T1 fonts
\usepackage{hyperref}       % hyperlinks
\usepackage{url}            % simple URL typesetting
\usepackage{booktabs}       % professional-quality tables
\usepackage{amsfonts}       % blackboard math symbols
\usepackage{amssymb}        % additional math symbols
\usepackage{amsmath}        % math environments
\usepackage{nicefrac}       % compact symbols for 1/2, etc.
\usepackage{microtype}      % microtypography
\usepackage{xcolor}         % colors
\usepackage{colortbl}       % cell colors in tables
\definecolor{ourhighlight}{RGB}{230,245,230}
\usepackage{graphicx}       % figures
\usepackage{algorithm}      % algorithm environment
\usepackage{algorithmic}    % algorithmic pseudocode
\usepackage{amsthm}         % theorem environments
\usepackage{enumitem}       % customizable lists
\usepackage{wrapfig}        % inline figures/tables
\usepackage{multirow}       % multirow cells in tables
\usepackage[most]{tcolorbox} % colored boxes for examples

\newcommand{\credit}{\textsc{Credit}}

\newtheorem{theorem}{Theorem}
\newtheorem{proposition}[theorem]{Proposition}
\newtheorem{corollary}[theorem]{Corollary}
\newtheorem{assumption}{Assumption}

\title{From Generic Correlation to Input-Specific Credit in On-Policy Self Distillation}

\author{%
  Guobin Shen$^{1}$ \quad Lei Huang$^{1}$ \quad Xiang Cheng$^{1}$ \quad Chenxiao Zhao$^{1}$ \\[2pt]
  \textbf{Jindong Li$^{2}$ \quad Dongcheng Zhao$^{2}$ \quad Xing Yu$^{1}$}\thanks{Correspondence to: \texttt{yuanshan2@xiaohongshu.com}, \texttt{floyed\_shen@outlook.com}.} \\[4pt]
  $^{1}$Xiaohongshu Inc. \quad $^{2}$Institute of Automation, Chinese Academy of Sciences
}

\begin{document}

\maketitle

\begin{abstract}
  On-policy self-distillation has emerged as a promising paradigm for post-training language models, in which the model conditions on environment feedback to serve as its own teacher, providing dense token-level rewards without external teacher models or step-level annotations. Despite its empirical success, what this reward actually measures and what kind of credit it assigns remain unclear. Under a posterior-compatibility interpretation of feedback conditioning, standard in the implicit-reward literature, we show that the self-distillation token reward is a Bayesian filtering increment whose trajectory sum is exactly the pointwise mutual information between the response and the feedback given the input. This pMI can be raised by input-specific reasoning or by input-generic shortcuts, so we further decompose the teacher log-probability along the input axis. Based on this analysis, we propose \credit{} (\textbf{C}ontrastive \textbf{RE}ward from \textbf{DI}s\textbf{T}illation), which isolates the input-specific component with a batch-contrastive baseline. At the sequence level, \credit{} is a teacher-side surrogate for a contrastive pMI objective that also penalizes responses remaining likely under unrelated inputs. Across coding, scientific reasoning, and tool-use benchmarks on two model families, \credit{} delivers the strongest aggregate performance at negligible additional compute.
\end{abstract}

%% ============================================================
%% 1. INTRODUCTION
%% ============================================================
\section{Introduction}
\label{sec:intro}

Reinforcement learning (RL) has become a dominant paradigm for post-training large language models (LLMs), enabling continued improvement on tasks with verifiable outcomes such as code generation, mathematical reasoning, and tool use~\citep{grpo,deepseek-r1,rlhf}. While RL demonstrably generalizes better than supervised fine-tuning~\citep{sft-mem-rl-gen}, a central bottleneck is \emph{credit assignment}: the model generates hundreds or thousands of tokens, yet receives only a single scalar reward at the end of the sequence. This sparse signal provides no information about which tokens contributed to success and which were irrelevant or harmful, leading to high gradient variance and inefficient learning.

Several lines of work address this bottleneck by providing denser reward signals. Process Reward Models (PRMs) train a separate model to score intermediate reasoning steps~\citep{prm800k,math-shepherd,prime,implicitprm}, but require step-level annotations or extensive Monte Carlo rollouts. On-policy distillation (OPD) uses a stronger teacher model to provide token-level supervision on the student's own trajectories~\citep{gkd,g-opd,opcd-ye}, offering dense on-policy signals but requiring access to a separate, often larger, teacher model whose quality upper-bounds the student.

On-policy Self-Distillation (OPSD) has recently emerged as a compelling alternative that addresses both limitations simultaneously. The key idea is to condition the model on environment feedback, such as ground-truth solutions, test results, or error messages, to form a \emph{self-teacher}, then distill this feedback-informed behavior back into the base policy. Multiple concurrent works have converged on this paradigm from different perspectives: reinforcement learning with rich feedback~\citep{sdpo}, rationalization of privileged information~\citep{opsd,pi-distill}, continual skill acquisition~\citep{sdft}, reasoning compression~\citep{opsdc}, and context internalization~\citep{opcd-ye}. The resulting token-level log-ratio between predictions with and without feedback provides a dense reward signal that is on-policy, requires no external teacher, and leverages the model's own in-context learning capabilities.

Despite this rapid empirical progress, the self-distillation reward remains only partially understood. Three questions remain open: \textit{RQ1}: What quantity does the self-distillation reward measure? \textit{RQ2}: Is this quantity biased toward input-generic correlations with feedback, rather than purely reflecting problem-specific credit for the current input? \textit{RQ3}: If so, can we derive a practical reward correction that reduces this input-generic bias while preserving useful problem-specific signal?

This work takes a step toward answering these questions:
\begin{itemize}[leftmargin=1.5em,itemsep=-1.5pt]
  \item \textit{RQ1.} Under posterior compatibility, the self-distillation token reward is a Bayesian filtering increment whose realized sum equals $\operatorname{pmi}_\pi(y; z \mid x)$, implying an implicit token-level Process Reward Model.
  \item \textit{RQ2.} We show that this signal exhibits an input-generic bias: it can reward both input-specific reasoning and generic response patterns that remain predictive of feedback across inputs.
  \item \textit{RQ3.} We propose \credit{} (\textbf{C}ontrastive \textbf{RE}ward from \textbf{DI}s\textbf{T}illation), a simple batch-contrastive correction that suppresses this input-generic component and improves aggregate performance with little additional compute.
\end{itemize}

%% ============================================================
%% 2. PRELIMINARIES
%% ============================================================
\section{Preliminaries}
\label{sec:prelim}

\begin{figure}[t]
  \centering
  \includegraphics[width=\linewidth]{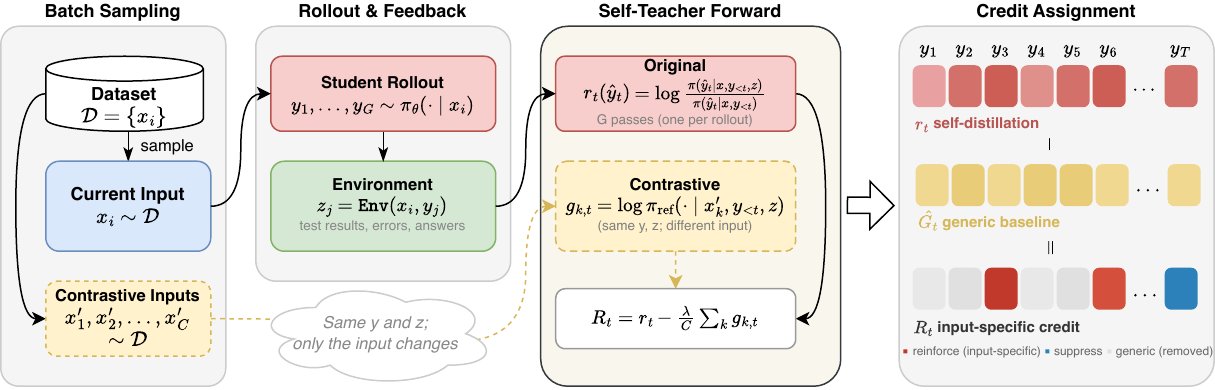}
  \caption{{Overview of \credit.} The self-teacher computes token-level rewards $r_t$ and a generic baseline $\hat{G}_t$ from contrastive inputs sharing the same response and feedback. Subtracting the baseline isolates input-specific credit.}
  \label{fig:overview}
\end{figure}

\paragraph{Problem setup.}
We consider a language model $\pi_\theta$ that generates a response $y = (y_1, \dots, y_T)$ to an input $x$ autoregressively: $\pi_\theta(y \mid x) = \prod_{t=1}^{T} \pi_\theta(y_t \mid x, y_{<t})$. The policy gradient for the RL objective takes the form
\begin{equation}
  \nabla_\theta J(\theta) = \mathbb{E}_{y \sim \pi_\theta(\cdot \mid x)} \left[ \sum_{t=1}^{T} \nabla_\theta \log \pi_\theta(y_t \mid x, y_{<t}) \cdot \hat{A}_t \right],
  \label{eq:pg}
\end{equation}
where $\hat{A}_t$ is the advantage of token $y_t$. Although Eq.~\eqref{eq:pg} updates at token-level granularity, only a scalar outcome reward $R(x, y)$ is typically available (e.g., pass/fail on unit tests). Standard methods such as GRPO estimate $\hat{A}_t \approx R(x,y) - b$ for all $t$, assigning equal credit to every token regardless of its actual contribution---a granularity mismatch that constitutes the credit-assignment bottleneck~\cite{zheng2025group}.

\paragraph{Self-distillation reward.}
\label{sec:log-ratio}
Many verifiable environments provide tokenized \emph{feedback} $z$ beyond the scalar reward, such as ground-truth answers, test cases, or compiler errors. Self-distillation~\citep{sdpo, opsd, opsdc} conditions the model on $z$ to form a self-teacher, resolving the granularity mismatch in Eq.~\eqref{eq:pg} without external teacher models or step-level annotations. At each position $t$, the dense reward is defined for every vocabulary token $\hat{y}_t \in \mathcal{V}$:
\begin{equation}
  r_t(\hat{y}_t) \;\triangleq\; \log \pi_{\text{ref}}(\hat{y}_t \mid x, y_{<t}, z) - \log \pi_\theta(\hat{y}_t \mid x, y_{<t}).
  \label{eq:sd-reward}
\end{equation}
The training objective minimizes a KL-style divergence from student to (stopgrad) teacher at each position, $\sum_t \mathrm{KL}\!\bigl(\pi_\theta(\cdot\mid x,y_{<t})\,\big\|\,\overline{\pi}_{\mathrm{ref}}(\cdot\mid x,y_{<t},z)\bigr)$, whose gradient is a policy-gradient update with per-vocabulary advantages $r_t(\hat{y}_t)$~\citep{sdpo}. In practice, $\pi_{\text{ref}}$ is a lagged copy of $\pi_\theta$ (e.g., via EMA) for stability, and the KL is restricted to the top-$K_\text{v}$ vocabulary tokens for efficiency.

%% ============================================================
%% 3. FROM SELF-DISTILLATION REWARD TO \credit
%% ============================================================
\section{From Self-Distillation Reward to \credit}
\label{sec:method}

\subsection{Self-Distillation Reward as Bayesian Filtering}
\label{sec:implicit-prm}

We begin by asking what the self-distillation reward $r_t(\hat{y}_t)$ (Eq.~\ref{eq:sd-reward}) actually measures. Consider the exact self-distillation setting $\pi_{\text{ref}} = \pi_\theta \equiv \pi$. Two separate forward passes of $\pi$ yield the policy's feedback-conditioned and unconditioned distributions, $\pi(\cdot \mid x, y_{<t}, z)$ and $\pi(\cdot \mid x, y_{<t})$. To interpret their log-ratio as an advantage under a Bayesian posterior, we adopt a posterior-compatibility interpretation that treats these two distributions as conditionals of a common joint model, in the spirit of implicit-reward interpretations used in DPO, PRIME, and related work~\citep{dpo,prime,implicitprm}. This is a modeling interpretation rather than a property guaranteed by training; Appendix~\ref{app:calibration} provides a controlled projected-compatibility check.

\begin{assumption}[Posterior compatibility]
  \label{asm:calibration}
  There exists a trajectory-level joint $P_\pi(y, z \mid x)$ such that, at every position $t$ and for all $\hat{y}_t$ and $z$, the two policy forward passes coincide with its conditionals:
  \[
    \pi(\hat{y}_t \mid x, y_{<t}, z) = P_\pi(\hat{y}_t \mid x, y_{<t}, z), \qquad \pi(\hat{y}_t \mid x, y_{<t}) = P_\pi(\hat{y}_t \mid x, y_{<t}).
  \]
  Equivalently, a single per-position joint $P_\pi(\hat{y}_t, z \mid x, y_{<t})$ exists for every $t$, and these per-position joints are consistent with a common trajectory-level joint $P_\pi(y, z \mid x)$.
\end{assumption}
Under Assumption~\ref{asm:calibration}, applying the chain rule to $P_\pi(\hat{y}_t, z \mid x, y_{<t})$ in two directions yields $\pi(\hat{y}_t \mid x, y_{<t}, z) \cdot P_\pi(z \mid x, y_{<t}) = P_\pi(z \mid x, y_{<t}, \hat{y}_t) \cdot \pi(\hat{y}_t \mid x, y_{<t})$. Taking the log-ratio gives our main identity:

\begin{theorem}[Self-distillation reward as Bayesian filtering]
  \label{thm:implicit-prm}
  Under Assumption~\ref{asm:calibration}, for any position $t$ and any $\hat{y}_t \in \mathcal{V}$,
  \begin{equation}
    r_t(\hat{y}_t) \;=\; \log \frac{\pi(\hat{y}_t \mid x, y_{<t}, z)}{\pi(\hat{y}_t \mid x, y_{<t})}
    \;=\; \log \frac{P_\pi(z \mid x, y_{<t}, \hat{y}_t)}{P_\pi(z \mid x, y_{<t})}
    \;=\; Q_t^{z}(\hat{y}_t,\, x) \;-\; V_{t-1}^{z}(x),
    \label{eq:implicit-prm}
  \end{equation}
\end{theorem}
\vspace{-0.5em}
where $Q_t^{z}(\hat{y}_t, x) \triangleq \log P_\pi(z \mid x, y_{<t}, \hat{y}_t)$ and $V_t^{z}(x) \triangleq \log P_\pi(z \mid x, y_{\leq t})$ are the action-value and state-value for predicting the feedback $z$ (we keep $x$ explicit and suppress the shared generated-prefix $y_{<t}$ in the notation). We borrow the $Q/V$ notation by analogy with control-as-inference~\citep{control-inference}; here $Q_t^z$ is the one-step conditional log-likelihood of $z$ given $\hat{y}_t$, not an integrated future return.

\paragraph{Counterfactual action vs.\ factual filtering increment.}
$r_t(\hat{y}_t)$ is defined for every candidate next token $\hat{y}_t \in \mathcal{V}$ and is used as the vocabulary-level advantage in the reverse-KL gradient: a \emph{counterfactual action value} asking how $\log P_\pi(z)$ would shift if $\hat{y}_t$ were selected. Along an actual rollout, only the sampled token $y_t$ is realized, and $r_t(y_t)$ reduces to the \emph{filtering increment} $\Delta V_t(x) \triangleq V_t^{z}(x) - V_{t-1}^{z}(x)$; only these realized increments telescope. We write $\hat{y}_t$ for counterfactual quantities and $y_t$ for realized ones throughout. For a fixed observed feedback value $z$, the prior expectation of $r_t$ is non-positive while the posterior expectation is non-negative, so reverse-KL distillation shifts mass away from tokens that make $z$ less predictable on average and toward posterior-supported tokens that make it more predictable. Summing realized increments along a trajectory gives the second consequence:

\begin{corollary}[Trajectory reward equals pointwise mutual information]
  \label{cor:trajectory-pmi}
  Under Assumption~\ref{asm:calibration}, for any realized trajectory $y$,
  \begin{equation}
    \sum_{t=1}^{T} r_t(y_t) \;=\; \log \frac{P_\pi(z \mid x, y)}{P_\pi(z \mid x)} \;=\; \operatorname{pmi}_\pi(y;\, z \mid x).
    \label{eq:trajectory-pmi}
  \end{equation}
\end{corollary}

Self-distillation thus admits an additive token-wise decomposition of the pointwise mutual information between response $y$ and feedback $z$ given $x$.

Averaging Eq.~\eqref{eq:trajectory-pmi} over the joint $(Y, Z) \mid X$ yields $\mathbb{E}_{(Y,Z) \mid X}\!\bigl[\sum_t r_t(Y_t)\bigr] = I_\pi(Y;\, Z \mid X)$, and by the chain rule $\mathbb{E}_{(Y_t, Z) \mid X, Y_{<t}}\!\bigl[r_t(Y_t)\bigr] = I_\pi(Y_t;\, Z \mid X,\, Y_{<t})$. This complements the fixed-$z$ view above: with $z$ fixed, the prior expectation of $r_t$ is non-positive, whereas averaging over jointly sampled responses and feedback recovers non-negative information content.

In this sense, the reward acts as an implicit token-level reward model for predicting feedback $z$: every vocabulary token receives an advantage equal to its marginal contribution to making $z$ more predictable, without additional models or step-level annotations~\citep{prm800k, math-shepherd}. When $z$ is a binary outcome (pass/fail), $\Delta V_t$ reduces to the change in log-probability of success, recovering potential-based reward shaping~\citep{reward-shaping}. This interpretation is exact in the self-distillation setting $\pi_{\text{ref}} = \pi_\theta$; with a lagged reference model, an additive teacher--student gap appears (Appendix~\ref{app:gap}). Appendix~\ref{app:interventional} develops a compatible interventional extension.

Together, these results answer \textit{RQ1} by characterizing what raw self-distillation measures. \textit{RQ2} then asks whether this signal is specific to the current input, or partly driven by more generic response patterns that would remain predictive of feedback across inputs.

\subsection{Separating Input-Specific Credit from Correlation}
\label{sec:decomposition}

\begin{figure}[t]
  \centering
  \includegraphics[width=\textwidth]{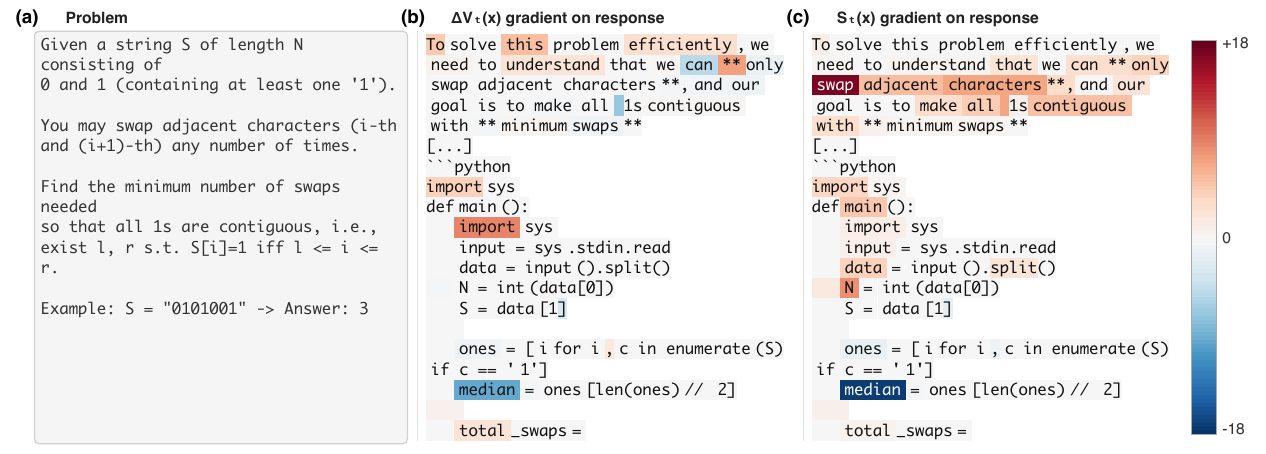}
  \caption{Token-level advantage on a response to problem~(a). (b)~Self-distillation reward $\Delta V_t(x)$: near-uniform, rewarding generic and problem-specific tokens alike. (c)~Input-specific signal $S_t(x)$ after \credit's decomposition: generic tokens are largely suppressed; problem-relevant tokens (red) are reinforced and the wrong algorithm choice is suppressed (blue). Additional examples in Appendix~\ref{app:visualizations}.}
  \label{fig:token-reward}
\end{figure}

The pMI interpretation above identifies what raw self-distillation rewards: informativeness about feedback $z$. But a high $\operatorname{pmi}_\pi(y; z \mid x)$ can arise from two distinct sources: \emph{input-specific reasoning}, where the response's tokens depend on the current input $x$ to predict $z$, or \emph{correlational shortcuts}, where the same tokens would remain predictive of $z$ under many plausible inputs. The raw self-distillation signal cannot distinguish between the two.

Figure~\ref{fig:token-reward} illustrates this on a coding problem. The response to a binary-string problem receives near-uniformly high $\Delta V_t(x)$: generic tokens such as \texttt{import} and filler phrases like ``To solve this problem efficiently'' are rewarded just as strongly as problem-specific algorithm choices. All of them increase $P(z \mid x, y_{\leq t})$, but only the latter requires understanding the input $x$.

The key observation is that input-specific reasoning should depend on $x$: if a token's advantage is equally large regardless of which problem is being solved, it reflects the model's prior rather than problem-specific reasoning. This suggests a simple counterfactual test: replace the current input $x$ with a different problem $x'$ from the same distribution and re-evaluate the advantage. If it persists, the token is input-generic (shortcut); if it vanishes, it is input-specific (reasoning).

Formally, let $\mathcal{D}$ denote the training distribution over inputs. Because the student term $\log \pi_\theta(\hat{y}_t \mid x, y_{<t})$ in Eq.~\eqref{eq:sd-reward} does not depend on the feedback $z$, any correlational shortcut from $z$ must live in the teacher term. We therefore decompose the \emph{teacher} log-probability of $\hat{y}_t$ along the input axis:
\begin{equation}
  \log \pi_{\text{ref}}(\hat{y}_t \mid x,\, y_{<t},\, z)
  \;=\; \underbrace{S_t(\hat{y}_t,\, x)}_{\text{input-specific}} \;+\; \underbrace{G_t(\hat{y}_t)}_{\text{input-generic}},
  \label{eq:decomposition}
\end{equation}
where
\begin{equation}
  G_t(\hat{y}_t) \;\triangleq\; \mathbb{E}_{x' \sim \mathcal{D}}\!\left[\log \pi_{\text{ref}}(\hat{y}_t \mid x',\, y_{<t},\, z)\right], \qquad
  S_t(\hat{y}_t,\, x) \;\triangleq\; \log \pi_{\text{ref}}(\hat{y}_t \mid x,\, y_{<t},\, z) - G_t(\hat{y}_t).
  \label{eq:st-gt}
\end{equation}
$S_t(\hat{y}_t, x)$ is large when the teacher's belief about $\hat{y}_t$ is concentrated on the current input (\emph{input-specific reasoning}) and near zero when it holds for arbitrary inputs given the same $z$ (\emph{shortcut}). $G_t(\hat{y}_t)$ captures the input-generic baseline shared across problems. Substituting Eq.~\eqref{eq:decomposition} into Eq.~\eqref{eq:sd-reward}, the self-distillation reward splits as
\begin{equation}
  r_t(\hat{y}_t;\, x) \;=\; S_t(\hat{y}_t, x) \;+\; G_t(\hat{y}_t) \;-\; \log \pi_\theta(\hat{y}_t \mid x,\, y_{<t}),
  \label{eq:sd-split}
\end{equation}
so that removing the input-generic teacher component $G_t$ leaves an input-specific reward. For the generated token $y_t$, we abbreviate $S_t(x) \equiv S_t(y_t, x)$ when the context is clear.

Figure~\ref{fig:token-reward}(c) shows the result on the same example: $S_t(x)$ selectively reinforces problem-specific vocabulary (``swap'', ``contiguous'') and suppresses the model's wrong algorithmic choice (``median''), while the generic tokens that dominated $\Delta V_t$ are largely suppressed.

\subsection{The \credit{} Algorithm}
\label{sec:credit}

Computing $G_t(\hat{y}_t)$ exactly requires averaging over $\mathcal{D}$, which is intractable. In practice, we estimate the input-generic baseline using $C$ other inputs $x'_1, \dots, x'_C$ sampled from the training batch ($x'_k \neq x$), and control the debiasing strength with a coefficient $\lambda \in [0, 1]$. For every $\hat{y}_t \in \mathcal{V}$:
\begin{equation}
  R_t(\hat{y}_t) \;=\; r_t(\hat{y}_t) \;-\; \frac{\lambda}{C}\sum_{k=1}^{C} \log \pi_{\text{ref}}(\hat{y}_t \mid x'_k, y_{<t}, z),
  \label{eq:credit}
\end{equation}
where $\hat{G}_t(\hat{y}_t) \triangleq \frac{1}{C}\sum_{k=1}^{C} \log \pi_{\text{ref}}(\hat{y}_t \mid x'_k, y_{<t}, z)$ is the estimated generic baseline. Setting $\lambda = 0$ recovers the standard self-distillation reward; increasing $\lambda$ progressively removes the input-generic teacher component. Algorithm~\ref{alg:credit} summarizes the full procedure.
% We find $\lambda \in [0.05, 0.2]$ works well across tasks (Section~\ref{sec:experiments}).

\begin{algorithm}[t]
  \caption{\credit: Contrastive REward from DIsTillation}
  \label{alg:credit}
  \begin{algorithmic}[1]
    \REQUIRE Batch $\{x_i\}_{i=1}^B$, policy $\pi_\theta$, reference model $\pi_{\text{ref}}$, contrastive count $C$, coefficient $\lambda$
    \FOR{each sample $x_i$ in the batch}
    \STATE $y_i \sim \pi_\theta(\cdot \mid x_i)$; \; $z_i \gets \texttt{Env}(x_i, y_i)$ \hfill $\triangleright$ Rollout \& feedback
    \STATE Sample $\{x'_k\}_{k=1}^C$ from the batch ($x'_k \neq x_i$) \hfill $\triangleright$ Contrastive inputs
    \STATE $\mathbf{f}_t \gets \log \pi_{\text{ref}}(\cdot \mid x_i, y_{<t}, z_i)$ for all $t$ \hfill $\triangleright$ Teacher logits (1 pass)
    \STATE $\mathbf{g}_{k,t} \gets \log \pi_{\text{ref}}(\cdot \mid x'_k, y_{<t}, z_i)$ for all $t, k$ \hfill $\triangleright$ Contrastive logits ($C$ passes)
    \STATE $R_t(\hat{y}_t) \gets \mathbf{f}_t[\hat{y}_t] - \frac{\lambda}{C}\textstyle\sum_{k} \mathbf{g}_{k,t}[\hat{y}_t] - \log \pi_\theta(\hat{y}_t \mid x_i, y_{<t})$ \hfill $\triangleright$ Eq.~\eqref{eq:credit}, $\forall\, \hat{y}_t \in \mathcal{V}$
    \ENDFOR
    \STATE Update $\pi_\theta$ via reverse-KL gradient using $\{R_t(\hat{y}_t)\}$ as vocabulary-level advantages
  \end{algorithmic}
\end{algorithm}

\credit{} requires $C$ additional forward passes through $\pi_{\text{ref}}$ per sample beyond standard self-distillation. All passes share the same response $y$ and feedback $z$, differing only in the input prefix, and can be fully parallelized. In practice, the vocabulary-level computation is restricted to the same top-$K_\text{v}$ tokens used by the reverse-KL loss, introducing no additional memory overhead. With $C = 1$ (the default in our experiments), the overhead is a single additional forward pass.

\subsection{Sequence-Level Contrastive pMI}
\label{sec:info-theory}

Corollary~\ref{cor:trajectory-pmi} identifies the quantity whose additive token-wise decomposition raw self-distillation provides: $\operatorname{pmi}_\pi(y; z \mid x)$. Because this pMI can be raised by either input-specific reasoning or input-generic shortcuts (\S\ref{sec:decomposition}), an input-specific objective should subtract the mismatched-input contribution, $\operatorname{pmi}_\pi(y; z \mid x) - \mathbb{E}_{x'}[\operatorname{pmi}_\pi(y; z \mid x')]$. A full-ratio contrastive reward that subtracts the entire self-distillation reward under contrastive $x'$ realizes this exactly when telescoped (Appendix~\ref{app:contrastive-pmi}). \credit{} (Eq.~\ref{eq:credit}) uses a \emph{teacher-side surrogate}, keeping only the contrastive teacher log-probability. In the exact self-distillation setting $\pi_{\text{ref}} = \pi_\theta \equiv \pi$, telescoping its realized rewards gives
\begin{equation}
  \sum_{t=1}^{T} R_t(y_t)
  \;=\; \operatorname{pmi}_\pi(y; z \mid x)
  \;-\; \lambda\, \mathbb{E}_{x'}\!\bigl[\operatorname{pmi}_\pi(y; z \mid x')\bigr]
  \;+\; \lambda\, \mathbb{E}_{x'}\!\bigl[-\log \pi(y \mid x')\bigr],
  \label{eq:credit-seq}
\end{equation}
where the first two terms are the ideal contrastive pMI (up to $\lambda$) and the third is an \emph{anti-genericity bonus}. Since $-\log \pi(y \mid x') \geq 0$ is the surprisal of $y$ under an unrelated input, this term is nonnegative and grows as $y$ becomes less likely under mismatched $x'$: boilerplate responses (high likelihood under random $x'$) receive little of this bonus, whereas input-specific responses receive more, further discouraging generic templates. Dropping the contrastive student term also saves compute, since the student log-probability is already computed by the reverse-KL loss.

\textbf{Token-level prior-contrastive surrogate.}
The sequence-level decomposition has a per-token analog. Define the \emph{prior-contrastive} quantity
\[
  \operatorname{pCMI}^{\mathcal D}(\hat{y}_t; x \mid y_{<t}, z) \;\triangleq\; \log \pi_{\text{ref}}(\hat{y}_t \mid x, y_{<t}, z) - \log \mathbb{E}_{x' \sim \mathcal{D}}[\pi_{\text{ref}}(\hat{y}_t \mid x', y_{<t}, z)],
\]
which contrasts the teacher's belief at the current input $x$ against the marginal teacher under random inputs from the data prior $\mathcal{D}$. This differs from the standard pointwise conditional mutual information~\citep{korbak} in that the marginalization uses the prior $\mathcal D$ rather than the posterior $P(x' \mid y_{<t}, z)$; the prior-contrastive form is what the algorithm actually computes from same-batch negatives, and it directly serves the goal of penalizing tokens that remain likely under randomly sampled inputs. The corresponding sample-based surrogate is $\hat{S}_t(\hat{y}_t, x) \triangleq \log \pi_{\text{ref}}(\hat{y}_t \mid x, y_{<t}, z) - \tfrac{1}{C}\sum_{k} \log \pi_{\text{ref}}(\hat{y}_t \mid x'_k, y_{<t}, z)$.

\begin{proposition}[Prior-contrastive surrogate]
  \label{prop:pcmi-bound}
  For any $C \geq 1$ with $x'_k \stackrel{\mathrm{i.i.d.}}{\sim} \mathcal{D}$, $\;\mathbb{E}_{\{x'_k\}}[\hat{S}_t(\hat{y}_t, x)] \geq \operatorname{pCMI}^{\mathcal D}(\hat{y}_t; x \mid y_{<t}, z)$ by Jensen's inequality, independent of Assumption~\ref{asm:calibration}.
\end{proposition}

The inequality above is one-sided, so $\hat{S}_t$ is a Jensen upper bound of $\operatorname{pCMI}^{\mathcal D}$ in expectation, not an unbiased estimator: tokens with high prior-contrastive informativeness are guaranteed to receive correspondingly high $\hat{S}_t$ in expectation, making the contrastive baseline input-aware rather than a generic control variate, but the bound does not by itself imply that low-pCMI shortcut tokens are suppressed. Shortcut suppression rests instead on Eq.~\eqref{eq:credit-seq}'s anti-genericity bonus and on our empirical visualizations (Fig.~\ref{fig:token-reward}, Appendix~\ref{app:visualizations}).

%% ============================================================
%% 4. EXPERIMENTS
%% ============================================================
\section{Experiments}
\label{sec:experiments}
% \vspace{-.5em}

\paragraph{Setup}
\label{sec:setup}
\vspace{-.5em}

We evaluate on LiveCodeBench v6~\cite{livecodebench} (avg@4), where feedback $z$ consists of test input-output pairs and runtime errors; SciKnowEval~\cite{sciknoweval} (Chemistry, Physics, Biology, Materials Science; avg@16), where $z$ is the ground-truth answer; and ToolAlpaca~\cite{toolalpaca} (tool-use; avg@16), where $z$ is the expected tool call. We use Qwen3-8B~\cite{qwen3} and OLMo-3-7B-Instruct~\cite{olmo} as base models, and compare GRPO~\cite{grpo} (scalar outcome reward), SDPO~\cite{sdpo} (\credit{} reduces to SDPO when $\lambda = 0$), and \credit{} ($\lambda = 0.1$, $C = 1$, fixed across all tasks). SDPO and \credit{} share identical training configurations and self-teacher contexts; \credit{} modifies only the reward computation, introducing no additional models or data. All experiments use the verl framework~\cite{verl} on 8 NVIDIA H20 GPUs. Dataset splits, evaluation protocol, and full hyperparameters are in Appendix~\ref{app:hparams}.

\subsection{Main Results}
\label{sec:main-results}
\vspace{-.5em}

We first evaluate on SciKnowEval and ToolAlpaca with both base models (Table~\ref{tab:main}). Self-distillation methods substantially outperform GRPO across all scientific reasoning domains on both models, confirming that dense credit from environment feedback is the primary driver of improvement. \credit{} delivers the strongest aggregate performance on both model families: it improves most clearly on the weaker OLMo model and on domains where SDPO is weaker or noisier, and largely matches SDPO where SDPO is already strong.

\begin{table}[ht]
  \centering
  \caption{Scientific reasoning and tool-use benchmarks (best avg@16 during training). Subscripts are bootstrap 95\% CI half-widths on per-prompt scores. \textbf{Bold}: best; \underline{underline}: second best.}
  \label{tab:main}
  \setlength{\tabcolsep}{1pt}
  \small
  \resizebox{\linewidth}{!}{
    \begin{tabular}{ l |ccccc|c| ccccc|c}
      \toprule
       & \multicolumn{6}{c|}{\textbf{Qwen3-8B}}       & \multicolumn{6}{c}{\textbf{OLMo-3-7B-Instruct}}                                                                                                                                                                                                                        \\
      \cmidrule(lr){2-7} \cmidrule(lr){8-13}
       & Chem.                                        & Phys.                                           & Bio.                                          & Mat.                                         & Tool                                          & \textbf{Avg.}    & Chem. & Phys. & Bio. & Mat. & Tool & \textbf{Avg.} \\
      \midrule
      Base
       & 36.2                                         & 59.2                                            & 25.8                                          & 58.6                                         & 57.5                                          & 47.5
       & 19.9                                         & 37.0                                            & 19.1                                          & 36.5                                         & 39.3                                          & 30.4                                                                  \\
      \;+ GRPO
       & 59.4$_{\scalebox{.7}{$\pm$5.7}}$             & 60.8$_{\scalebox{.7}{$\pm$8.0}}$                & 29.4$_{\scalebox{.7}{$\pm$7.1}}$              & 72.5$_{\scalebox{.7}{$\pm$8.2}}$             & \underline{69.7}$_{\scalebox{.7}{$\pm$10.3}}$ & 58.4
       & 46.4$_{\scalebox{.7}{$\pm$5.0}}$             & 60.3$_{\scalebox{.7}{$\pm$8.4}}$                & 29.3$_{\scalebox{.7}{$\pm$7.6}}$              & 72.7$_{\scalebox{.7}{$\pm$8.5}}$             & \underline{64.4}$_{\scalebox{.7}{$\pm$11.6}}$ & 54.6                                                                  \\
      \;+ SDPO
       & \underline{65.6}$_{\scalebox{.7}{$\pm$5.3}}$ & \textbf{71.8}$_{\scalebox{.7}{$\pm$7.9}}$       & \underline{51.8}$_{\scalebox{.7}{$\pm$11.3}}$ & \underline{72.6}$_{\scalebox{.7}{$\pm$9.1}}$ & 68.0$_{\scalebox{.7}{$\pm$10.7}}$             & \underline{66.0}
       & \underline{66.8}$_{\scalebox{.7}{$\pm$5.8}}$ & \underline{71.4}$_{\scalebox{.7}{$\pm$7.8}}$    & \underline{42.0}$_{\scalebox{.7}{$\pm$11.7}}$ & \underline{77.2}$_{\scalebox{.7}{$\pm$7.9}}$ & 61.5$_{\scalebox{.7}{$\pm$10.8}}$             & \underline{63.8}                                                      \\
      \rowcolor{ourhighlight}
      \;+ \credit
       & \textbf{66.4}$_{\scalebox{.7}{$\pm$5.6}}$    & \underline{71.7}$_{\scalebox{.7}{$\pm$8.0}}$    & \textbf{52.4}$_{\scalebox{.7}{$\pm$11.4}}$    & \textbf{74.2}$_{\scalebox{.7}{$\pm$8.6}}$    & \textbf{70.4}$_{\scalebox{.7}{$\pm$10.2}}$    & \textbf{67.0}
       & \textbf{69.2}$_{\scalebox{.7}{$\pm$5.7}}$    & \textbf{72.3}$_{\scalebox{.7}{$\pm$8.7}}$       & \textbf{54.1}$_{\scalebox{.7}{$\pm$11.4}}$    & \textbf{77.9}$_{\scalebox{.7}{$\pm$7.5}}$    & \textbf{67.4}$_{\scalebox{.7}{$\pm$10.2}}$    & \textbf{68.2}                                                         \\
      \bottomrule
    \end{tabular}
  }
\end{table}

An interesting pattern emerges across domains. Where self-distillation already provides large gains over GRPO (e.g., Biology on Qwen3-8B), \credit{} offers modest further improvement: the dense reward is already sufficiently informative. Conversely, where self-distillation struggles (e.g., tool use on OLMo, where SDPO underperforms GRPO), \credit{} recovers and surpasses both methods, suggesting that contrastive debiasing is most valuable precisely when the raw self-distillation reward contains more noise than signal. This is consistent with our theoretical analysis: simpler feedback structures produce a larger input-generic component $G_t$, which \credit{} removes.

\begin{figure}[ht]
  \centering
  \includegraphics[width=\linewidth]{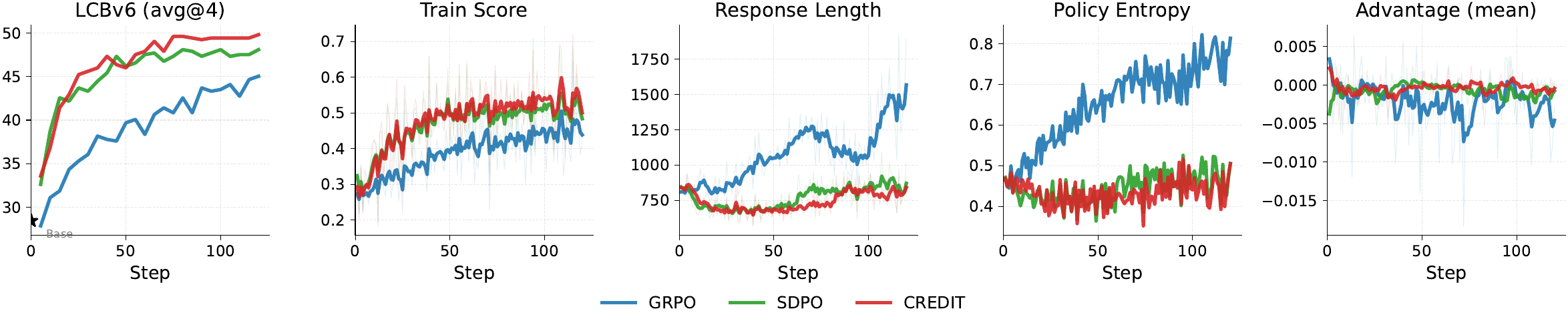}
  \vspace{-2em}
  \caption{LiveCodeBench v6 training dynamics (Qwen3-8B). \credit{} achieves the highest eval score while maintaining shorter responses and more stable entropy and advantages than both SDPO and GRPO.}
  \vspace{-1em}
  \label{fig:lcb-curves}
\end{figure}

We additionally evaluate on LiveCodeBench v6~\cite{livecodebench} with Qwen3-8B, where the environment provides rich structured feedback (test results, runtime errors) rather than a scalar reward. \credit{} again outperforms both baselines (Figure~\ref{fig:lcb-curves}, left), and also converges faster: it reaches 45\% accuracy within roughly 25 training steps, compared to about 40 steps for SDPO and around 120 steps for GRPO. The training dynamics reveal that GRPO's response length more than doubles over training while its entropy steadily rises, indicating increasingly verbose and uncertain generation. In contrast, both self-distillation methods produce shorter, more concise responses with stable entropy. Among the two, \credit{} exhibits more stable advantages throughout training, suggesting that contrastive debiasing yields a better-calibrated reward signal. Beyond the example in Figure~\ref{fig:token-reward}, we provide additional token-level reward visualizations on representative LCB problems in Appendix~\ref{app:visualizations}, where $S_t$ consistently concentrates credit on problem-specific reasoning steps while largely suppressing generic tokens.

\subsection{Ablation Studies}
\label{sec:ablation}
\vspace{-.5em}

We ablate the key hyperparameters of \credit{} on LiveCodeBench v6 with Qwen3-8B, where the rich feedback structure makes differences between configurations most visible.

\begin{wraptable}{r}{0.57\linewidth}
  \vspace{-1.5em}
  \centering
  \small
  \caption{{$C\!=\!1$ suffices}: \credit{} performance is stable across contrastive counts while per-step compute grows with $C$. Time ratios are measured relative to the SDPO baseline ($C\!=\!0$). Subscripts on the LCBv6 column report the absolute percentage-point improvement over the SDPO baseline (not confidence intervals).}
  \label{tab:ablation-c}
  \begin{tabular}{c|ccc}
    \toprule
    $C$          & LCBv6 (\%)                            & Train Score & Time / step         \\
    \midrule
    0 \,({SDPO}) & 48.1                                  & 0.707       & 467s (1.00$\times$) \\
    \midrule
    1            & 49.8$_{\scalebox{.7}{+1.7}}$          & 0.719       & 472s (1.01$\times$) \\
    2            & 49.6$_{\scalebox{.7}{+1.5}}$          & 0.715       & 481s (1.03$\times$) \\
    8            & \textbf{50.6}$_{\scalebox{.7}{+2.5}}$ & 0.719       & 607s (1.30$\times$) \\
    16           & 49.4$_{\scalebox{.7}{+1.3}}$          & 0.719       & 765s (1.64$\times$) \\
    \bottomrule
  \end{tabular}
  \vspace{-.5em}
\end{wraptable}
\textbf{Contrastive count $C$.}
We also ablate $C$, the number of contrastive samples used to estimate $\hat{G}_t$, and find that a single sample already suffices (Table~\ref{tab:ablation-c}). All values of $C \geq 1$ reach the same peak train score and comparable LCB accuracy, while per-step compute grows roughly linearly with $C$, reaching 1.64$\times$ the SDPO baseline at $C = 16$ with no accuracy gain. Larger $C$ does reduce the variance of the $\hat{G}_t$ estimate ($S_t$ std drops from 1.10 to 0.97 as $C$ grows from 1 to 16), but this variance reduction does not translate into better learning. The reason is that $C = 1$ is already an unbiased estimator of the log-space baseline term $\mathbb{E}_{x'}[\log \pi_{\text{ref}}(\hat{y}_t \mid x', y_{<t}, z)]$ under a fixed policy (Proposition~\ref{prop:pcmi-bound}); larger $C$ primarily reduces estimator variance of a quantity that is already averaged over many training steps. The reverse-KL distillation loss further attenuates this variance through two mechanisms specific to its token-level structure: (i)~the top-$K_v$ truncation restricts the loss to the teacher's most likely tokens, where the contrastive log-probability is most stable, and (ii)~the EMA reference model evolves slowly relative to the policy update, so $\hat{G}_t$ enters as a low-frequency baseline rather than a noisy per-update target. We therefore use $C = 1$ throughout. The observed wall-clock overhead over SDPO is only about 1\%, because training time is dominated by autoregressive rollout and the contrastive forward pass shares the response prefix and feedback with the original teacher pass, allowing the two to be batched together at negligible additional cost.

\begin{figure}[ht]
  \centering
  \includegraphics[width=\linewidth]{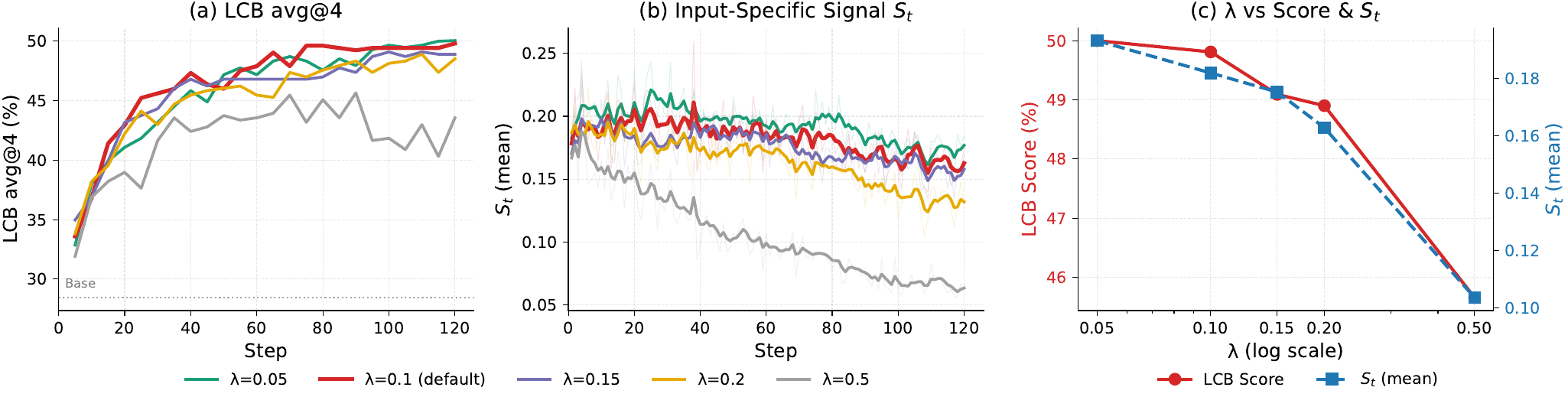}
  \vspace{-2em}
  \caption{{Over-debiasing collapses the input-specific signal.} (a)~LCBv6 score for varying $\lambda$. (b)~$S_t$ over training: larger $\lambda$ causes rapid decay. (c)~Both LCB and $S_t$ decrease with $\lambda$ in lockstep.}
  \label{fig:ablation-lambda}
  \vspace{-1em}
\end{figure}

\textbf{Debiasing strength $\lambda$.}
$\lambda$ controls how aggressively \credit{} removes the input-generic component from the reward. Figure~\ref{fig:ablation-lambda}(a) shows that $\lambda \in [0.05, 0.1]$ performs best, with monotonic degradation at larger values. The mechanism is visible in the input-specific signal $S_t$ (Figure~\ref{fig:ablation-lambda}(b)): all $\lambda$ values start at similar levels ($\sim$0.18), but $S_t$ decays to 0.07 for $\lambda = 0.5$. This creates a feedback loop: excessive debiasing weakens the policy gradient, producing a more generic policy whose $S_t$ shrinks further. Figure~\ref{fig:ablation-lambda}(c) confirms that both LCB score and $S_t$ decrease with $\lambda$ in lockstep. We fix $\lambda = 0.1$ for all other experiments.

\textbf{Self-teacher context.}
Our main experiments use Qwen3-8B without thinking mode (the student response contains no thinking trace). Enabling thinking---i.e., letting the student produce a thinking trace and including that trace in the self-teacher's context---improves all three methods (Figure~\ref{fig:ablation-think-bar}), with \credit{} benefiting the most (+8.8 points). The improvement is consistent across methods, suggesting that richer context gives the self-teacher more information to assign dense credit, an effect that compounds with \credit's contrastive debiasing.

\begin{wrapfigure}{r}{0.35\linewidth}
  \vspace{-1.0em}
  \centering
  \includegraphics[width=\linewidth]{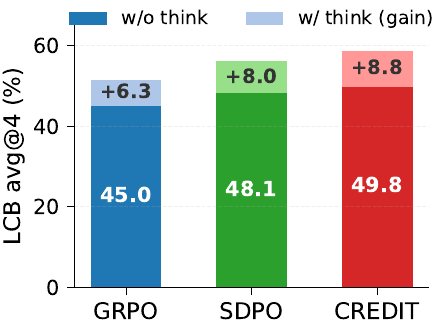}
  \vspace{-1.0em}
  \caption{{Think mode on LCBv6.} Enabling thinking improves all methods; \credit{} gains the most.}
  \label{fig:ablation-think-bar}
  \vspace{-1em}
\end{wrapfigure}
Given the thinking trace, what should the self-teacher context contain? We compare three variants for \credit{} in Figure~\ref{fig:ablation-context}: \emph{think + solution} (the full context), \emph{solution only} (thinking trace removed), and \emph{think only} (solution removed). Solution-only achieves the best stable performance (Figure~\ref{fig:ablation-context}(a)). The think-only variant initially learns fastest, peaking at 57.1\% by step 35, but then collapses to 26.7\%. The train score (Figure~\ref{fig:ablation-context}(b)) mirrors this trajectory, ruling out evaluation noise, and suggests that the thinking trace alone provides a strong but unstable supervision signal. We hypothesize that without the grounding provided by the solution, the self-teacher's predictions become increasingly self-referential, amplifying noise until the policy diverges. The input-specific signal $S_t$ (Figure~\ref{fig:ablation-context}(c)) is highest when both components are present and lowest for think-only, consistent with the solution providing a more stable anchor for credit assignment.

\begin{figure}[ht]
  \centering
  \includegraphics[width=\linewidth]{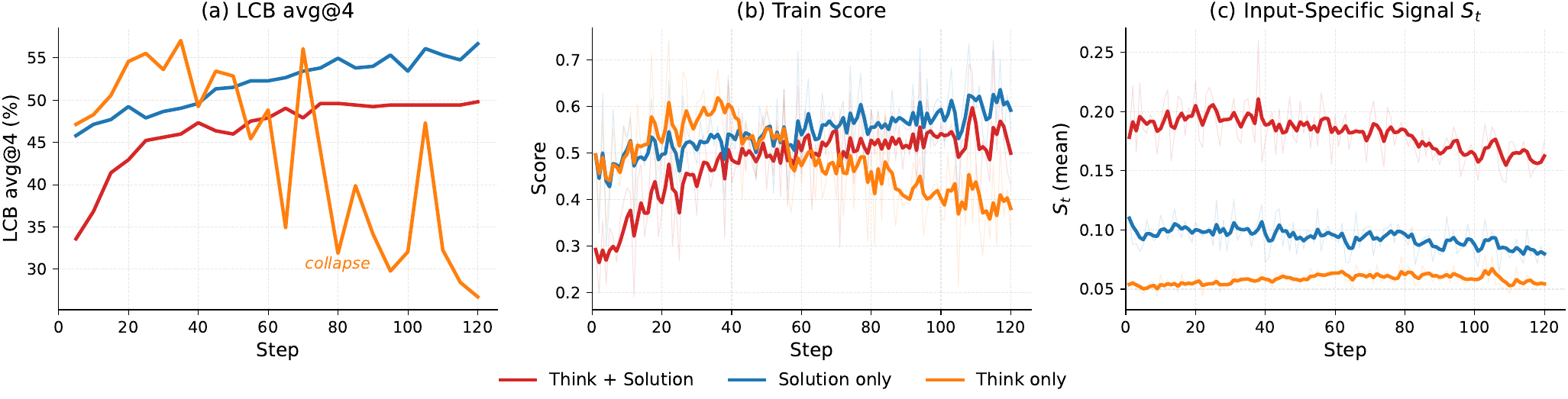}
  \vspace{-1em}
  \caption{\textbf{Self-teacher context ablation (\credit{}, w/ think).} Think-only collapses after an initial peak; solution-only is most stable. Train score (b) confirms the collapse is not an evaluation artifact.}
  \label{fig:ablation-context}
  \vspace{-1em}
\end{figure}

\vspace{-.5em}
\section{Related Work}
\label{sec:related}
\vspace{-.5em}

\textbf{On-policy self-distillation.}
Knowledge distillation~\citep{kd,minillm} trains a student to match a teacher's output distribution; on-policy variants~\citep{gkd} apply dense token-level supervision on the student's own trajectories to avoid distribution mismatch. Earlier self-improvement work such as STaR~\citep{star} and self-rewarding language models~\citep{yuan-self-reward} lets the model learn from its own successful outputs, but operates at the sequence level via filtering or judging rather than at the token level. A recent family of works replaces the external teacher with a \emph{self-teacher} conditioned on privileged information~\citep{vapnik-lupi,snell-context-distill}, such as ground-truth answers, test results, or tool outputs, yielding token-level supervision without an external model: SDPO~\citep{sdpo} for RL with rich environment feedback, OPSD~\citep{opsd} and $\pi$-Distill~\citep{pi-distill} for rationalization of privileged information, OPSDC~\citep{opsdc} for reasoning compression, SDFT~\citep{sdft} for continual skill acquisition, OPCD~\citep{opcd-ye} for context internalization, and G-OPD~\citep{g-opd} for generalizing beyond the teacher via reward scaling.

\credit{} takes a complementary perspective: rather than proposing a new self-distillation variant, we analyze the reward structure common to all of them. Several of these works use or motivate the log-ratio as dense credit, but none decompose the signal along the input axis or identify the input-generic component targeted by \credit{}. More concretely, SDPO~\citep{sdpo} and $\pi$-Distill~\citep{pi-distill} use the self-teacher log-ratio directly without isolating how much of that signal is input-generic; OPSD's~\citep{opsd} token-level clipping stabilizes large per-token log-ratios regardless of input, whereas \credit{} specifically targets tokens whose teacher likelihood remains high under mismatched inputs; and G-OPD~\citep{g-opd} globally rescales the teacher–student ratio, whereas \credit{} changes the reward content by subtracting an input-contrastive baseline. \credit{} modifies only the reward computation and composes with any of these frameworks.

\textbf{Process and implicit rewards.}
Process reward models assign reward to intermediate reasoning steps rather than full sequences: PRM800K~\citep{prm800k} uses human annotations, Math-Shepherd~\citep{math-shepherd} automates annotation via Monte Carlo rollouts, and PRIME~\citep{prime} and Yuan et al.~\citep{implicitprm} derive implicit process rewards from outcome supervision, though building reliable PRMs remains hard~\citep{lessons-prm}. A parallel line of work derives rewards implicit in policy ratios: DPO~\citep{dpo} interprets the policy log-ratio as a reward, and the self-distillation log-ratio admits a similar reading. Our Theorem~\ref{thm:implicit-prm} sharpens this view: the token reward is a Bayesian filtering increment whose trajectory sum is a pointwise mutual information (Corollary~\ref{cor:trajectory-pmi}), and \credit{} refines it by removing the input-generic component that inflates credit for tokens merely correlated with feedback.

\textbf{Credit assignment.}
Classical credit assignment methods such as GAE~\citep{gae} decompose reward along the \emph{temporal} axis (early vs.\ late tokens), and potential-based reward shaping~\citep{reward-shaping} adds shaped rewards without altering the optimal policy. The $S_t$/$G_t$ decomposition in \credit{} instead operates along the \emph{input} axis, separating credit specific to the current input from credit any input would receive. The control-as-inference framework~\citep{control-inference,korbak} motivates the $Q$/$V$ interpretation we use in Theorem~\ref{thm:implicit-prm}: treating policy optimization as posterior inference makes the connection between log-ratios and action-value advantages transparent. Our contrastive baseline, estimated from other inputs in the batch, provides a simple Monte Carlo approximation to the input-generic component $G_t$, which we connect to pointwise conditional mutual information through a one-sided Jensen bound (Proposition~\ref{prop:pcmi-bound}).

\section{Conclusion}
\label{sec:conclusion}
\vspace{-.5em}

We presented \credit{}, a simple reward correction for on-policy self-distillation. Under a posterior-compatibility interpretation of feedback conditioning, the self-distillation token reward is a Bayesian filtering increment whose trajectory sum equals the pointwise mutual information between response and feedback given the input; since pMI can be raised by input-generic shortcuts as well as input-specific reasoning, \credit{} subtracts a batch-contrastive baseline that makes the learning objective a teacher-side surrogate for a contrastive pMI plus an anti-genericity bonus on mismatched-input surprisal, delivering the strongest aggregate performance across coding, scientific reasoning, and tool-use benchmarks at negligible compute overhead. Our experiments are limited to 7--8B models and structured-feedback tasks, and the contrastive baseline assumes sufficient batch diversity and pairs unrelated $x'$ with $z$ in a way that may push the teacher prompt out of distribution; a systematic study of negative-sampling strategies, dynamically adapted $\lambda$, extensions to multi-turn agent trajectories, and applications to other policy-ratio implicit rewards are left to future work.

\bibliography{refs}
\bibliographystyle{unsrt}

%%%%%%%%%%%%%%%%%%%%%%%%%%%%%%%%%%%%%%%%%%%%%%%%%%%%%%%%%%%%

\newpage
\appendix

\begin{center}
  \noindent\rule{\linewidth}{3pt}\\[6pt]
  {\Large\bfseries From Generic Correlation to Input-Specific Credit\\[0.25em] in On-Policy Self Distillation\\[0.4em] Supplementary Material}\\[6pt]
  \noindent\rule{\linewidth}{1pt}
\end{center}
\vspace{1em}

\noindent\textbf{Table of Contents}
\vspace{0.3em}
\hrule
\vspace{0.5em}
\noindent
\hyperref[app:hparams]{\ref{app:hparams}}\hspace{1em}Training Hyperparameters\dotfill\pageref{app:hparams}\\[0.3em]
\hyperref[app:proofs]{\ref{app:proofs}}\hspace{1em}Proofs\dotfill\pageref{app:proofs}\\[0.3em]
\hyperref[app:gap]{\ref{app:gap}}\hspace{1em}Approximation Gap when $\pi_{\text{ref}} \neq \pi_\theta$\dotfill\pageref{app:gap}\\[0.3em]
\hyperref[app:contrastive-pmi]{\ref{app:contrastive-pmi}}\hspace{1em}Sequence-Level Contrastive pMI: Full Derivation\dotfill\pageref{app:contrastive-pmi}\\[0.3em]
\hyperref[app:calibration]{\ref{app:calibration}}\hspace{1em}Empirical Support for Assumption~\ref{asm:calibration} at the Answer Position\dotfill\pageref{app:calibration}\\[0.3em]
\hyperref[app:interventional]{\ref{app:interventional}}\hspace{1em}Toward Interventional Credit\dotfill\pageref{app:interventional}\\[0.3em]
\hyperref[app:visualizations]{\ref{app:visualizations}}\hspace{1em}Additional Token-Level Reward Visualizations\dotfill\pageref{app:visualizations}
\vspace{0.5em}
\hrule
\vspace{1.5em}

%% ============================================================
%% A. HYPERPARAMETERS
%% ============================================================
\section{Training Hyperparameters}
\label{app:hparams}

% \paragraph{Splits and evaluation protocol.}
We follow the data-preparation and evaluation pipeline of SDPO~\cite{sdpo} directly across all three benchmarks, so that relative comparisons between GRPO, SDPO, and \credit{} remain on identical footing. \textbf{LiveCodeBench v6} uses the 131 programming questions released between February and May 2025, with a LeetCode-style public/private test setup: a $50\%$ random subset of LCB's private tests provides in-training feedback $z$ and the full private set is used for validation (avg@4). \textbf{SciKnowEval} uses the Level-3 reasoning subsets for Biology, Chemistry, Materials Science, and Physics, with the SDPO train/test split (avg@16). \textbf{ToolAlpaca} uses its standard 4046/68 train/test split (avg@16). Validation runs every 5 training steps; we use the same selection protocol across all three methods and index training progress by optimization step. SDPO and \credit{} use \emph{identical} self-teacher contexts $(x, y_{<t}, z)$; when environment feedback includes a successful rollout from the current training group (e.g.\ a passing program on LCB), it enters $z$ for both methods.

% \paragraph{Evaluation CIs.}
The bootstrap 95\% CI half-widths reported as subscripts in Table~\ref{tab:main} are computed by resampling evaluation prompts with replacement (2000 resamples). The sampling unit is the prompt, and each prompt's score is the mean of $K$ samples ($K{=}16$ on SciKnowEval and ToolAlpaca, $K{=}4$ on LiveCodeBench); the CI therefore quantifies evaluation-prompt sampling variance in the reported mean. The same evaluation protocol is applied uniformly across GRPO, SDPO, and \credit{}, so relative comparisons between methods are on identical footing.

% \paragraph{Hyperparameters.}
Table~\ref{tab:hparams} lists the full training hyperparameters. All three methods (GRPO, self-distillation, \credit) share the same configuration except for the advantage computation. \credit{} adds only two hyperparameters: the contrastive count $C$ and the debiasing coefficient $\lambda$.

\begin{table}[h]
  \centering
  \caption{Training hyperparameters for LiveCodeBench (LCB) and generalization tasks (SciKnowEval + ToolUse). Self-distillation and \credit{} share identical configurations; \credit{} adds only $C$ and $\lambda$.}
  \label{tab:hparams}
  \setlength{\tabcolsep}{1pt}
  % \small
  \begin{tabular}{lcc}
    \toprule
    \textbf{Hyperparameter}            & \textbf{LCB}                                                                           & \textbf{SciKnowEval / ToolUse} \\
    \midrule
    Optimizer                          & \multicolumn{2}{c}{Adam}                                                                                                \\
    Learning rate                      & \multicolumn{2}{c}{$1 \times 10^{-6}$}                                                                                  \\
    LR warmup steps                    & 0                                                                                      & 10                             \\
    Training epochs                    & 30                                                                                     & 3                              \\
    Batch size (problems)              & 32                                                                                     & 32                             \\
    Rollouts per problem (group size)  & 8                                                                                      & 8                              \\
    PPO mini-batch size                & 8                                                                                      & 32                             \\
    Training rollout temperature       & \multicolumn{2}{c}{1.0 (top-$p\!=\!1.0$, top-$k\!=\!-1$)}                                                               \\
    Evaluation rollout temperature     & \multicolumn{2}{c}{0.6 (top-$p\!=\!0.95$)}                                                                              \\
    Validation rollouts                & 4                                                                                      & 16                             \\
    Validation frequency               & \multicolumn{2}{c}{every 5 training steps}                                                                              \\
    Max sequence length                & 18944                                                                                  & 18944                          \\
    GRPO clip ratio                    & \multicolumn{2}{c}{0.2}                                                                                                 \\
    Divergence to teacher              & reverse KL ($\alpha\!=\!1.0$)                                                          & JSD ($\alpha\!=\!0.5$)         \\
    Top-$K_\text{v}$ vocabulary for KL & 20 (teacher's top-20 tokens)                                                           & 100 (teacher's top-100 tokens) \\
    Reference model $\pi_{\text{ref}}$ & \multicolumn{2}{c}{EMA of $\pi_\theta$, update rate $0.01$}                                                             \\
    \midrule
    \credit{} contrastive count $C$    & 1                                                                                      & 1                              \\
    \credit{} coefficient $\lambda$    & 0.1                                                                                    & 0.1                            \\
    Contrastive input sampling         & \multicolumn{2}{c}{uniform from same batch, excluding own prompt}                                                       \\
    Contrastive prompt concat          & \multicolumn{2}{c}{$x'_k$ replaces $x$ in the user-turn prefix; $y_{<t}, z$ unchanged}                                  \\
    \midrule
    Hardware                           & \multicolumn{2}{c}{1 node $\times$ 8 NVIDIA H20 GPUs}                                                                   \\
    Framework                          & \multicolumn{2}{c}{verl~\cite{verl} with FSDP}                                                                          \\
    \bottomrule
  \end{tabular}
\end{table}

\subsection{Self-Teacher Context Examples}
\label{app:prompts}

We show the context that the self-teacher $\pi_{\text{ref}}(\cdot \mid x, y_{<t}, z)$ sees when re-evaluating the student's response. The teacher's input concatenates the original prompt, a correct solution from the batch (if available), and environment feedback $z$, followed by an instruction to re-solve the problem. The student's original response $y$ is placed in the assistant role; the teacher then re-evaluates $y$'s log-probabilities under this enriched context. Colors: {\color{blue!70!black}original prompt}, {\color{red!70!black}feedback $z$}, {\color{teal}student response $y$} (re-evaluated by teacher).

\definecolor{sysblue}{HTML}{2B5797}
\definecolor{respgreen}{HTML}{1A7A4C}
\definecolor{fbredcolor}{HTML}{B91C1C}

\begin{tcolorbox}[colback=gray!3, colframe=gray!50, title={\small LiveCodeBench --- Self-teacher context}, fonttitle=\bfseries\small, boxrule=0.5pt, arc=2pt, left=4pt, right=4pt, top=2pt, bottom=2pt]
  \small\ttfamily
  {\color{sysblue}\textbf{[User]}} You are a coding expert\ldots write a correct Python program\ldots\\
  Given rectangles on a 2D plane\ldots find the minimum y such that the total area above equals below\ldots\\[3pt]
  Correct solution:\\
  \textrm{\textit{(a successful rollout from the same batch, if available)}}\\
  \texttt{def separateSquares(s): \ldots \# correct implementation}\\[3pt]
  {\color{fbredcolor}Previous assessment:}\\
  {\color{fbredcolor}Test 1: Wrong Answer}\\
  {\color{fbredcolor}Input: [[26,30,2],[1,23,1]] \quad Output: 5 \quad Expected: 4}\\
  {\color{fbredcolor}Test 2: Runtime Error}\\
  {\color{fbredcolor}ZeroDivisionError: division by zero}\\
  {\color{fbredcolor}Line 15 in separateSquares (Solution.py)}\\[3pt]
  Now solve this problem step by step.\\[3pt]
  {\color{respgreen}\textbf{[Assistant]}} {\color{respgreen}\textit{(student's original response $y$, re-evaluated by teacher)}}\\
  {\color{respgreen}def separateSquares(squares): \ldots}
\end{tcolorbox}

\begin{tcolorbox}[colback=gray!3, colframe=gray!50, title={\small SciKnowEval --- Self-teacher context}, fonttitle=\bfseries\small, boxrule=0.5pt, arc=2pt, left=4pt, right=4pt, top=2pt, bottom=2pt]
  \small\ttfamily
  {\color{sysblue}\textbf{[System]}} Given a question and four options, select the right answer. Respond in the following format: <reasoning>\ldots</reasoning> <answer>\ldots</answer>\\
  For the answer, only output the letter (A, B, C, or D).\\[3pt]
  {\color{sysblue}\textbf{[User]}} Which of the following correctly describes the function of S- or Se-methyltransferases?\\
  A: They regulate\ldots \; B: They (de)methylate thiol and selenol metabolites\\
  C: They break down\ldots \; D: They catalyze\ldots\\[3pt]
  {\color{fbredcolor}Correct answer: B}\\[3pt]
  Now solve this problem step by step.\\[3pt]
  {\color{respgreen}\textbf{[Assistant]}} {\color{respgreen}\textit{(student's original response $y$)}}\\
  {\color{respgreen}<reasoning>\ldots</reasoning>}\\
  {\color{respgreen}<answer>A</answer>}
\end{tcolorbox}

\begin{tcolorbox}[colback=gray!3, colframe=gray!50, title={\small ToolAlpaca --- Self-teacher context}, fonttitle=\bfseries\small, boxrule=0.5pt, arc=2pt, left=4pt, right=4pt, top=2pt, bottom=2pt]
  \small\ttfamily
  {\color{sysblue}\textbf{[User]}} Your task is to answer the user's question using available tools.\\
  You have access to: Axolotl --- getRandomAxolotlImage, searchAxolotlImages\ldots\\[3pt]
  Question: Hey, can you show me a random picture of an axolotl?\\[3pt]
  {\color{fbredcolor}Actions mismatch: predicted [searchAxolotlImages],}\\
  {\color{fbredcolor}expected [getRandomAxolotlImage]}\\[3pt]
  Now solve this problem step by step.\\[3pt]
  {\color{respgreen}\textbf{[Assistant]}} {\color{respgreen}\textit{(student's original response $y$)}}\\
  {\color{respgreen}Thought: The user wants a random axolotl image\ldots}\\
  {\color{respgreen}Action: searchAxolotlImages}\\
  {\color{respgreen}Action Input: \{"color": "wild"\}}
\end{tcolorbox}

%% ============================================================
%% B. PROOFS
%% ============================================================
\section{Proofs}
\label{app:proofs}

We collect explicit proofs of the main theoretical results stated in Section~\ref{sec:method}. All identities are pointwise and make no distributional assumptions beyond those stated in each claim.

\paragraph{Proof of Theorem~\ref{thm:implicit-prm}.}
By Assumption~\ref{asm:calibration}, there is a joint $P_\pi(\hat{y}_t, z \mid x, y_{<t})$ whose two conditionals coincide with the policy's feedback-conditioned and unconditioned forward passes. Applying the product rule to this joint in two directions gives
\begin{align*}
  P_\pi(\hat{y}_t,\, z \mid x, y_{<t})
   & = P_\pi(\hat{y}_t \mid x, y_{<t}, z) \cdot P_\pi(z \mid x, y_{<t})          \\
   & = P_\pi(z \mid x, y_{<t}, \hat{y}_t) \cdot P_\pi(\hat{y}_t \mid x, y_{<t}).
\end{align*}
Equating and substituting the two marginal identities from Assumption~\ref{asm:calibration} yields
\begin{equation*}
  \pi(\hat{y}_t \mid x, y_{<t}, z) \cdot P_\pi(z \mid x, y_{<t})
  \;=\; P_\pi(z \mid x, y_{<t}, \hat{y}_t) \cdot \pi(\hat{y}_t \mid x, y_{<t}).
\end{equation*}
Dividing both sides by $\pi(\hat{y}_t \mid x, y_{<t}) \cdot P_\pi(z \mid x, y_{<t})$ (positive by Assumption~\ref{asm:calibration} with respect to the joint's support) and taking logs gives the chain of equalities in Eq.~\eqref{eq:implicit-prm} after substituting the definitions of $Q_t^{z}$ and $V_{t-1}^{z}$.
The sign consequences stated in Section~\ref{sec:implicit-prm} follow by taking expectations of $r_t$ under the two policy distributions and matching against the definition of KL divergence:
\begin{align*}
  \mathbb{E}_{\pi(\cdot \mid x, y_{<t})}[r_t]
   & = \mathbb{E}_{\pi(\cdot \mid x, y_{<t})}\!\left[\log \frac{\pi(\cdot \mid x, y_{<t}, z)}{\pi(\cdot \mid x, y_{<t})}\right]
  = -\, D_\mathrm{KL}\!\bigl(\pi(\cdot \mid x, y_{<t}) \,\|\, \pi^z\bigr) \;\leq\; 0,                                           \\
  \mathbb{E}_{\pi^z}[r_t]
   & = D_\mathrm{KL}\!\bigl(\pi^z \,\|\, \pi(\cdot \mid x, y_{<t})\bigr) \;\geq\; 0.
\end{align*}
\hfill$\square$

\paragraph{Proof of Corollary~\ref{cor:trajectory-pmi}.}
Specializing Theorem~\ref{thm:implicit-prm} to the realized token $y_t$ (i.e., setting $\hat{y}_t = y_t$) gives
\begin{equation*}
  r_t(y_t) \;=\; \log P_\pi(z \mid x, y_{<t}, y_t) - \log P_\pi(z \mid x, y_{<t})
  \;=\; V_t^{z}(x) - V_{t-1}^{z}(x),
\end{equation*}
where $V_t^{z}(x) \triangleq \log P_\pi(z \mid x, y_{\leq t})$. The per-step realized reward is thus the one-step filtering increment of the state-value $V^z$ along the trajectory. Summing over $t = 1, \ldots, T$ telescopes:
\begin{equation*}
  \sum_{t=1}^{T} r_t(y_t) \;=\; V_T^{z}(x) - V_0^{z}(x) \;=\; \log P_\pi(z \mid x, y) - \log P_\pi(z \mid x) \;=\; \operatorname{pmi}_\pi(y;\, z \mid x),
\end{equation*}
using $V_0^{z}(x) = \log P_\pi(z \mid x)$ (empty prefix) and $V_T^{z}(x) = \log P_\pi(z \mid x, y)$ (complete response), together with the definition of pointwise mutual information. \hfill$\square$

The expectation-level identity $\mathbb{E}_{(Y, Z) \mid X}\!\bigl[\sum_t r_t(Y_t)\bigr] = I_\pi(Y; Z \mid X)$ in Section~\ref{sec:implicit-prm} follows by taking expectations of both sides of Eq.~\eqref{eq:trajectory-pmi} over $(Y, Z) \mid X$ and matching against the definition of conditional mutual information; the per-token chain-rule version $\mathbb{E}_{(Y_t, Z) \mid X, Y_{<t}}[r_t(Y_t)] = I_\pi(Y_t;\, Z \mid X,\, Y_{<t})$ follows analogously by applying the chain rule of mutual information to the sum.

\paragraph{Proof of Proposition~\ref{prop:pcmi-bound}.}
By linearity of expectation over the $C$ i.i.d.\ contrastive inputs,
\begin{equation*}
  \mathbb{E}_{\{x'_k\}}\!\left[\hat{S}_t(\hat{y}_t, x)\right]
  \;=\; \log \pi_{\text{ref}}(\hat{y}_t \mid x, y_{<t}, z) \;-\; \mathbb{E}_{x' \sim \mathcal{D}}\!\left[\log \pi_{\text{ref}}(\hat{y}_t \mid x', y_{<t}, z)\right],
\end{equation*}
since each summand in $\tfrac{1}{C} \sum_k \log \pi_{\text{ref}}(\hat{y}_t \mid x'_k, y_{<t}, z)$ has the same marginal mean. By Jensen's inequality applied to the concave function $\log$,
\begin{equation*}
  \mathbb{E}_{x' \sim \mathcal{D}}\!\left[\log \pi_{\text{ref}}(\hat{y}_t \mid x', y_{<t}, z)\right]
  \;\leq\; \log \mathbb{E}_{x' \sim \mathcal{D}}\!\left[\pi_{\text{ref}}(\hat{y}_t \mid x', y_{<t}, z)\right].
\end{equation*}
Negating and adding $\log \pi_{\text{ref}}(\hat{y}_t \mid x, y_{<t}, z)$ to both sides gives
\begin{align*}
  \mathbb{E}_{\{x'_k\}}\!\left[\hat{S}_t(\hat{y}_t, x)\right]
   & \;\geq\; \log \pi_{\text{ref}}(\hat{y}_t \mid x, y_{<t}, z) - \log \mathbb{E}_{x' \sim \mathcal{D}}\!\left[\pi_{\text{ref}}(\hat{y}_t \mid x', y_{<t}, z)\right] \\
   & \;=\; \operatorname{pCMI}^{\mathcal D}(\hat{y}_t; x \mid y_{<t}, z).
\end{align*}
The bound holds for every $C \geq 1$ and does not invoke Assumption~\ref{asm:calibration}: it is a statement about the policy $\pi_{\text{ref}}$ and the data distribution $\mathcal{D}$ only. \hfill$\square$

%% ============================================================
%% C. APPROXIMATION GAP
%% ============================================================
\section{\texorpdfstring{Approximation Gap when $\pi_{\text{ref}} \neq \pi_\theta$}{Approximation Gap when pi\_ref differs from pi\_theta}}
\label{app:gap}

Under posterior compatibility, Theorem~\ref{thm:implicit-prm} is an exact identity in the self-distillation setting $\pi_{\text{ref}} = \pi_\theta = \pi$. In practice, a lagged reference model is used (e.g., an EMA copy or a periodically synchronized snapshot). We characterize the resulting approximation gap.

When $\pi_{\text{ref}} \neq \pi_\theta$, the reward $r_t(\hat{y}_t)$ decomposes as:
\begin{align}
  r_t(\hat{y}_t)
   & = \log \pi_{\text{ref}}(\hat{y}_t \mid x, y_{<t}, z) - \log \pi_\theta(\hat{y}_t \mid x, y_{<t}) \nonumber                                                                   \\
   & = \underbrace{\log \frac{\pi_{\text{ref}}(\hat{y}_t \mid x, y_{<t}, z)}{\pi_{\text{ref}}(\hat{y}_t \mid x, y_{<t})}}_{\text{(i) implicit advantage under } \pi_{\text{ref}}}
  \;+\; \underbrace{\log \frac{\pi_{\text{ref}}(\hat{y}_t \mid x, y_{<t})}{\pi_\theta(\hat{y}_t \mid x, y_{<t})}}_{\text{(ii) teacher-student gap}}.
  \label{eq:gap-decomposition}
\end{align}

Term~(i) is the implicit PRM under the reference model: assuming $\pi_{\text{ref}}$ is posterior-compatible (Assumption~\ref{asm:calibration}), Theorem~\ref{thm:implicit-prm} applied to $\pi_{\text{ref}}$ gives $Q_t^{\text{ref}}(\hat{y}_t, x) - V_{t-1}^{\text{ref}}(x)$, where $Q_t^{\text{ref}}$ and $V_t^{\text{ref}}$ are the action-value and state-value under $\pi_{\text{ref}}$. This term carries all feedback-derived credit.

Term~(ii) is the log-ratio between the reference and current policy evaluated \emph{without} feedback $z$. It depends on $x$ and $\hat{y}_t$, but not on $z$, so it is not a feedback-derived credit signal. When $\pi_{\text{ref}}$ closely tracks $\pi_\theta$ (via EMA, periodic synchronization, or trust-region constraints), it is bounded by the max of the two one-sided R\'enyi-$\infty$ divergences at each position:
\begin{equation}
  \left|\text{term (ii)}\right| \;\leq\; \max\!\Bigl(
  D_\infty\!\bigl(\pi_{\text{ref}}(\cdot \mid x, y_{<t}) \,\big\|\, \pi_\theta(\cdot \mid x, y_{<t})\bigr),\;
  D_\infty\!\bigl(\pi_\theta(\cdot \mid x, y_{<t}) \,\big\|\, \pi_{\text{ref}}(\cdot \mid x, y_{<t})\bigr)
  \Bigr),
\end{equation}
since $D_\infty(p \| q) = \sup_y \log p(y)/q(y)$ upper-bounds only one side of $\log \pi_{\text{ref}}/\pi_\theta$; both divergences are simultaneously small whenever $\pi_{\text{ref}}$ and $\pi_\theta$ are close.

\credit{}'s contrastive correction applies \emph{only} to the feedback-conditioned teacher log-probability (Eq.~\ref{eq:credit}), leaving term~(ii) unchanged. Term~(ii) is therefore the same standard distillation/regularization term that appears in SDPO and is shared by SDPO and \credit{}; the two methods differ only in how they treat term~(i).

\paragraph{Summary.} With a lagged reference model, the implicit-PRM interpretation holds under $\pi_{\text{ref}}$ rather than $\pi_\theta$, and an additional $z$-independent student/reference log-ratio term appears. This additional term carries no feedback-derived credit and is present identically in both SDPO and \credit{}; the contrastive baseline in \credit{} debiases the feedback-dependent teacher signal without interacting with it.

%% ============================================================
%% D. SEQUENCE-LEVEL CONTRASTIVE pMI DERIVATION
%% ============================================================
\section{Sequence-Level Contrastive pMI: Full Derivation}
\label{app:contrastive-pmi}

We derive Eq.~\eqref{eq:credit-seq} in the main text, which expresses the sequence-level sum of \credit{} rewards in terms of a contrastive pointwise mutual information plus an anti-genericity bonus on mismatched-input surprisal. Throughout this appendix we work under Assumption~\ref{asm:calibration} and the exact self-distillation setting $\pi_{\text{ref}} = \pi_\theta \equiv \pi$; the lagged case adds only the $z$-independent gap of Appendix~\ref{app:gap}, which does not interact with the derivation.

\paragraph{Full-ratio contrastive reward.}
Define the input-dependent self-distillation reward
\[
  r_t(\hat{y}_t;\, x') \;\triangleq\; \log \pi(\hat{y}_t \mid x', y_{<t}, z) \;-\; \log \pi(\hat{y}_t \mid x', y_{<t}),
\]
and the full-ratio contrastive reward
\[
  R^{\text{full}}_t(\hat{y}_t) \;\triangleq\; r_t(\hat{y}_t;\, x) \;-\; \lambda\, \mathbb{E}_{x' \sim \mathcal{D}}[r_t(\hat{y}_t;\, x')].
\]
Telescoping the realized-token reward along the trajectory $y$ and applying Corollary~\ref{cor:trajectory-pmi} to each input gives
\begin{equation*}
  \sum_{t=1}^T R^{\text{full}}_t(y_t) \;=\; \operatorname{pmi}_\pi(y;\, z \mid x) \;-\; \lambda\, \mathbb{E}_{x'}\!\bigl[\operatorname{pmi}_\pi(y;\, z \mid x')\bigr],
\end{equation*}
which is exactly the ideal contrastive pMI $\operatorname{pmi}_\pi(y; z \mid x) - \lambda\, \mathbb{E}_{x'}[\operatorname{pmi}_\pi(y; z \mid x')]$ of Section~\ref{sec:info-theory}.

\paragraph{\credit{} as a teacher-side surrogate.}
The \credit{} reward in Eq.~\eqref{eq:credit} drops the contrastive student term, keeping only the feedback-conditioned teacher log-probability under $x'$:
\[
  R_t(\hat{y}_t) \;=\; r_t(\hat{y}_t;\, x) \;-\; \lambda\, \mathbb{E}_{x'}[\log \pi(\hat{y}_t \mid x', y_{<t}, z)].
\]
Telescoping the realized-token reward gives
\begin{align*}
  \sum_{t=1}^T R_t(y_t)
   & = \operatorname{pmi}_\pi(y;\, z \mid x) \;-\; \lambda\, \mathbb{E}_{x'}\!\bigl[\log \pi(y \mid x', z)\bigr].
\end{align*}
By Bayes under posterior compatibility (Assumption~\ref{asm:calibration}), $\pi(y \mid x', z) = \pi(y \mid x') \cdot P_\pi(z \mid x', y) / P_\pi(z \mid x')$, so $\log \pi(y \mid x', z) = \log \pi(y \mid x') + \operatorname{pmi}_\pi(y;\, z \mid x')$ as an exact identity (the $y$-independent normalizer $-\log P_\pi(z \mid x')$ is absorbed into the pMI definition). Substituting yields
\[
  \sum_{t=1}^T R_t(y_t) \;=\; \operatorname{pmi}_\pi(y; z \mid x) \;-\; \lambda\, \mathbb{E}_{x'}\!\bigl[\operatorname{pmi}_\pi(y; z \mid x')\bigr] \;+\; \lambda\, \mathbb{E}_{x'}\!\bigl[-\log \pi(y \mid x')\bigr],
\]
recovering Eq.~\eqref{eq:credit-seq} exactly, with no residual additive constant.

\paragraph{Interpretation of the anti-genericity bonus.}
The extra term $\lambda\, \mathbb{E}_{x'}[-\log \pi(y \mid x')]$ is the expected surprisal of $y$ under unrelated inputs. Because $-\log \pi(y \mid x') \ge 0$, the term is nonnegative and grows as $y$ becomes \emph{less} likely under random $x'$: responses that only the current $x$ could plausibly produce receive a larger bonus, while generic responses the model would produce under arbitrary inputs receive less. \credit{} is thus slightly more aggressive at suppressing input-generic boilerplate than the ideal contrastive pMI alone.

%% ============================================================
%% E. EMPIRICAL VERIFICATION OF ASSUMPTION 1
%% ============================================================
\section{Empirical Support for Assumption 1 at the Answer Position}
\label{app:calibration}

Theorem~\ref{thm:implicit-prm} and Corollary~\ref{cor:trajectory-pmi} rest on Assumption~\ref{asm:calibration}, which treats the policy's two forward passes $\pi(\cdot \mid x, y_{<t})$ and $\pi(\cdot \mid x, y_{<t}, z)$ as conditionals of a common joint $P_\pi(\hat{y}_t, z \mid x, y_{<t})$. This is a modeling interpretation rather than a derived property. Here we provide empirical support in a controlled setting: the 4-dimensional answer-letter subspace of a multiple-choice benchmark, where the \emph{projected compatibility condition}---whether the student distribution lies in the convex hull of the four teacher distributions---can be checked by a small convex-feasibility problem. We do not attempt to certify full-vocabulary calibration at arbitrary prefixes; this is a necessary-consequence check on a particular projection.

\paragraph{Setup.}
We use SciKnowEval~\citep{sciknoweval} at Level~3 on the Materials Science domain, a 4-option MCQ with $z \in \mathcal{Z} = \{A, B, C, D\}$. We draw a gold-balanced subset of 100 problems per model (25 per gold letter, no filtering by model correctness) and test two models: Qwen3-8B (nothink mode, matching our main experiments) and OLMo-3-7B-Instruct. For each problem we construct chat prompts mirroring SDPO's reprompt format; for the student context we feed the original prompt, for the teacher context for each $z \in \mathcal{Z}$ we additionally inject ``Previous assessment: The correct answer is $\backslash$boxed\{$z$\}. Now solve this problem step by step.'' The assistant turn is forced to the answer block via the prefix \texttt{<answer>\textbackslash n}, so the next-token position is pinned to the answer letter.

\paragraph{Measurement.}
For each problem and each context we run one forward pass, extract the next-token probability vector, and project it onto the 4-letter subspace by summing over all token ids whose decoded form strips to a given letter and renormalizing. This yields $s \in \Delta^4$ (student) and $T \in \mathbb{R}^{4 \times 4}$ (teacher rows indexed by $z$). On this subspace, Assumption~\ref{asm:calibration} becomes: does there exist $P \in \Delta^{|\mathcal{Z}|}$ with $s = T^\top P$? We solve
\begin{equation*}
  \hat{P} \;=\; \arg\min_{P \in \Delta^{|\mathcal{Z}|}} \|\,s - T^\top P\,\|_1
\end{equation*}
using SLSQP with simplex constraints, and report the residual $\|s - T^\top \hat{P}\|_1$. A residual of zero means the projected compatibility condition holds exactly on this subspace; note that this is strictly weaker than full-vocabulary compatibility, because projecting onto 4 dimensions discards the $|\mathcal{V}| - 4$-dimensional detail the assumption also constrains.

\paragraph{Baselines and prerequisites.}
Three side measurements guard the interpretation. (i)~\emph{Letter-token mass} $\sum_{\ell \in \mathcal{Z}} \sum_{\tau \in \mathrm{Ids}(\ell)} \pi(\tau \mid \cdot)$: if the forcing prefix works, this is close to $1$. (ii)~\emph{Teacher fidelity} $T_{z \to z}$: if below $\approx 0.3$ the LLM is ignoring the $z$ hint and the test cannot distinguish compatibility from trivial teacher collapse. (iii)~\emph{Uniform-$P$ baseline} $\|s - T^\top \mathbf{u}\|_1$ with $\mathbf{u} = (1/|\mathcal{Z}|, \dots, 1/|\mathcal{Z}|)$: bounds the residual under the uninformative mixture. A small LS residual is only meaningful when the uniform baseline is large.

\paragraph{Primary metric: self-consistency.}
Because we do not filter by gold correctness, we report a model-internal consistency rate that does not depend on ground truth: \emph{does $\arg\max_z \hat{P}_z$ match $\arg\max_{\ell} s_\ell$?} When teacher fidelity is high (teachers place most of their mass on their respective letters), compatibility implies that the recovered posterior $\hat{P}$ should usually align with the student's own answer distribution. A high self-consistency rate is therefore a joint signal that both prerequisites hold (fidelity non-trivial, mixture exists) and that the fit is semantically coherent, not a numerical coincidence.

\begin{figure}[t]
  \centering
  \includegraphics[width=0.9\linewidth]{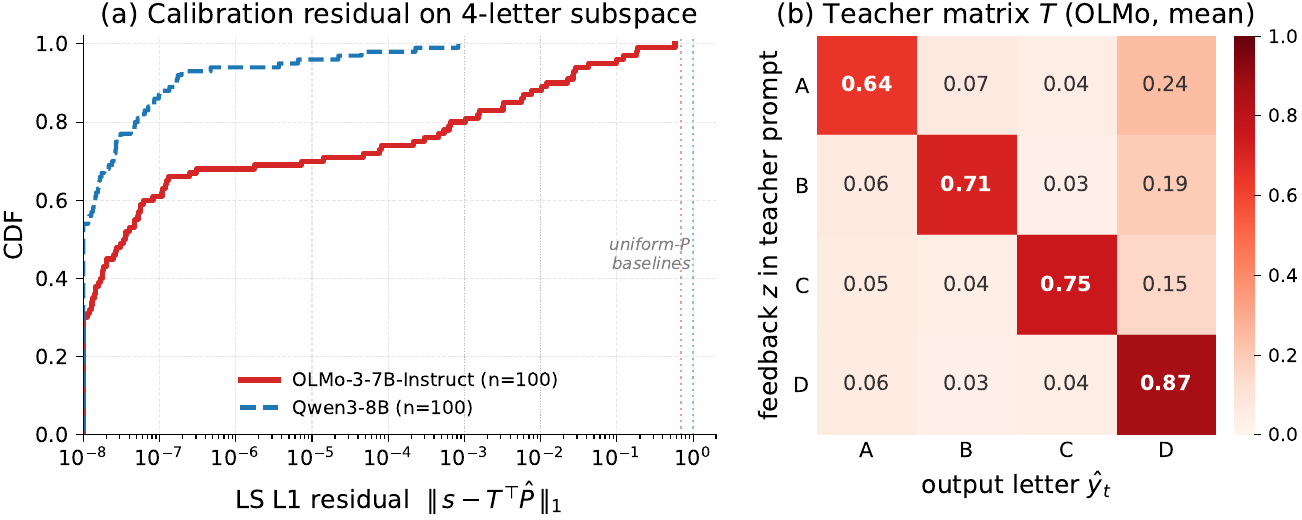}
  \caption{\textbf{Projected compatibility check for Assumption~\ref{asm:calibration} at the answer-letter position.} (a) CDF of the LS residual $\|s - T^\top \hat{P}\|_1$ on the 4-letter subspace, for 100 gold-balanced SciKnowEval Material problems per model. Qwen3-8B (dashed blue) achieves residual $< 10^{-3}$ on 100\% of records; OLMo-3-7B-Instruct (solid red) on 80\%, with the remaining tail tracking records where teacher fidelity fails (see (b)). Dotted vertical lines mark the median uniform-$P$ baseline ($\approx 0.68$ OLMo, $\approx 0.99$ Qwen). (b) Teacher matrix $T$ averaged over OLMo records: diagonal (bold) shows teacher fidelity to the $z$ hint, off-diagonal $D$ column shows a residual bias toward D that accounts for the fidelity tail.}
  \label{fig:calibration}
\end{figure}

\begin{table}[t]
  \centering
  \caption{Summary of the projected compatibility check for Assumption~\ref{asm:calibration} on SciKnowEval Material (100 gold-balanced records per model). LS residuals on the 4-letter subspace are several orders of magnitude below both the uniform-$P$ baseline and the paper-relevant $10^{-3}$ threshold on the majority of records; self-consistency $\arg\max(\hat{P}) = \arg\max(s)$ confirms the fit is semantically coherent.}
  \label{tab:calibration}
  \small
  \begin{tabular}{lcc}
    \toprule
    Metric                                           & Qwen3-8B             & OLMo-3-7B-Instruct   \\
    \midrule
    \# records                                       & 100                  & 100                  \\
    Gold balance (A / B / C / D)                     & 25 / 25 / 25 / 25    & 25 / 25 / 25 / 25    \\
    Letter-token mass (median)                       & 1.000                & 1.000                \\
    Teacher fidelity (median, worst $z$)             & 0.986                & 0.862                \\
    Teacher fidelity (median, best $z$)              & 0.999                & 0.990                \\
    \midrule
    LS L1 residual (median)                          & $9.0 \times 10^{-9}$ & $3.5 \times 10^{-8}$ \\
    Fraction $\|s - T^\top\hat{P}\|_1 < 10^{-3}$     & \textbf{100\%}       & \textbf{80\%}        \\
    Uniform-$P$ L1 baseline (median)                 & 0.99                 & 0.68                 \\
    Uniform / LS ratio (median)                      & $9.0 \times 10^{7}$  & $1.5 \times 10^{7}$  \\
    \midrule
    Self-consistency $\arg\max \hat{P} = \arg\max s$ & \textbf{100\%}       & \textbf{96\%}        \\
    $\hat{P}$ mass on student's argmax (median)      & 0.997                & 0.754                \\
    \bottomrule
  \end{tabular}
\end{table}

\paragraph{Results.}
Figure~\ref{fig:calibration} and Table~\ref{tab:calibration} summarize the outcome. Prerequisites are satisfied: letter-token mass has median $1.00$ for both models, and teacher fidelity medians sit at $0.86$--$0.99$, so teachers genuinely respond to the $z$ hint and span the 4-simplex (uniform-$P$ baseline $\approx 0.68$--$0.99$). Under these conditions, the projected compatibility residual is effectively zero for the large majority of records: on Qwen3-8B it falls below $10^{-3}$ on \textbf{every} record (median $9 \times 10^{-9}$, max $8 \times 10^{-4}$), and on OLMo on 80\% of records. The OLMo tail is not evidence against the assumption: it tracks records where teacher fidelity fails (the q25 of per-$z$ fidelity is $0.2$--$0.6$ on OLMo versus $0.9$--$1.0$ on Qwen, and Figure~\ref{fig:calibration}(b) shows a residual bias toward $D$ in the teacher matrix), i.e., exactly the records where the LLM itself ignores the $z$ hint and the projected compatibility question is ill-posed because the four teacher rows collapse rather than spanning the simplex. Self-consistency is 100\% (Qwen) / 96\% (OLMo): the LS-recovered $\hat{P}$ exactly or almost always picks the same letter the student does, and its median mass on that letter is $\geq 0.75$. Secondary gold-based metrics (not used to select problems) show that $\arg\max \hat{P}$ tracks student argmax even when the student is wrong, not ground truth, which is the exact semantics of the implicit posterior $P(z \mid x, y_{<t})$ in this subspace.

\paragraph{Scope.}
Compatibility on the projection is a \emph{necessary consequence} of Assumption~\ref{asm:calibration}, not the assumption itself. This test is a \emph{projected compatibility} check at one terminal position under a forcing prompt, on a 4-dim semantic subspace: passing it rules out a worst-case failure of the assumption on this controlled projection but does not certify full-vocabulary compatibility at arbitrary interior prefixes, does not bear on Jensen-gap effects that our per-token pCMI analysis (Proposition~\ref{prop:pcmi-bound}) already acknowledges as one-sided, and does not speak to the credit assignment actually realized during training at non-terminal positions. What it does provide is a quantitative anchor for reading $\hat{P}$ as a semantically meaningful implicit posterior on this projection. We view this as empirical support for the \emph{modeling interpretation} of Assumption~\ref{asm:calibration} under the SDPO reprompt format, in the same sense that DPO's reward-as-log-ratio is supported by its empirical effectiveness rather than derived from training dynamics; it is \emph{not} a mechanistic claim about what the reward signal attributes credit to during training.

%% ============================================================
%% F. TOWARD INTERVENTIONAL CREDIT
%% ============================================================
\section{Toward Interventional Credit}
\label{app:interventional}

Our main theory reads the self-distillation reward as Bayesian filtering on a feedback variable $z$ (Theorem~\ref{thm:implicit-prm}): $r_t$ measures how much a candidate token $\hat{y}_t$ updates the posterior over $z$. \credit{} sits squarely in this \emph{informational, input-conditional} frame---it is a surrogate that sharpens the input-conditional signal in $r_t^{z}$, not a causal-credit estimator; in general, $z$ need not itself be a downstream consequence of the current token. This appendix develops a complementary \emph{interventional} reading at the level of the target, not of \credit{}: we define an ideal counterfactual target based on a final success outcome $O$, and characterize sufficient conditions under which the feedback-conditioned reward's ordering over candidate tokens coincides with this target's. Under a feedback-sufficiency assumption, ordering is preserved (Prop.~\ref{prop:rank-preserve}), and under a one-sided-witness sub-case the gap collapses to an action-independent constant (Cor.~\ref{cor:one-sided}). We emphasize upfront that this appendix establishes sufficient conditions, not a causal identification of \credit's signal; an empirical tie-down under a setting that non-trivially satisfies these conditions (e.g., binary verifier correctness on a code or tool-use task) is left to future work.

\paragraph{Outcome variable and ideal interventional credit.}
Let $h_t = (x, y_{<t})$ and let $\hat{y}_t$ denote a candidate token at position $t$. Introduce a final-success variable $O \in \{0, 1\}$ corresponding to a downstream verifier ($O = 1$ iff the completed rollout sampled under $\pi$ after position $t$ is judged correct: MCQ letter matches gold, tool-use sequence matches reference, code passes tests, and so on). The \emph{ideal interventional credit} is the log-uplift in success probability under an atomic intervention at position $t$:
\begin{equation}
  r_t^{\mathrm{cf}}(\hat{y}_t;\, h_t) \;\triangleq\; \log \frac{P_\pi\!\left(O = 1 \,\middle|\, \mathrm{do}(\hat{Y}_t = \hat{y}_t),\, h_t\right)}{P_\pi(O = 1 \mid h_t)},
  \label{eq:rcf}
\end{equation}
where $\mathrm{do}(\hat{Y}_t = \hat{y}_t)$ forces position $t$ to the candidate and lets subsequent tokens be sampled under $\pi$. Exact estimation of \eqref{eq:rcf} would require branching intervention rollouts at every prefix for every candidate---infeasible as a dense training signal. We therefore treat $r_t^{\mathrm{cf}}$ as an ideal target and ask how close $r_t^{z}$ comes to it.

\paragraph{Identification via a success-conditioned teacher.}
Assume (i)~\emph{consistency}: after $\mathrm{do}(\hat{Y}_t = \hat{y}_t)$, subsequent tokens follow the unmodified $\pi$; (ii)~\emph{sequential ignorability}: given $h_t$, no unobserved variable jointly determines $\hat{Y}_t$ and $O$; (iii)~\emph{positivity}: $\pi(\hat{y}_t \mid h_t) > 0$. Define the success-conditioned teacher
\begin{equation*}
  \pi^{+}(\hat{y}_t \mid h_t) \;\triangleq\; P_\pi(\hat{Y}_t = \hat{y}_t \mid h_t,\, O = 1).
\end{equation*}
Under (i)--(iii), do-calculus reduces intervention to conditioning and Bayes' rule rearranges the ratio:
\begin{equation}
  r_t^{\mathrm{cf}}(\hat{y}_t;\, h_t) \;=\; \log \frac{\pi^{+}(\hat{y}_t \mid h_t)}{\pi(\hat{y}_t \mid h_t)} \;=\; \operatorname{pmi}_\pi(\hat{Y}_t = \hat{y}_t \,;\, O = 1 \mid h_t).
  \label{eq:rcf-pmi}
\end{equation}

\paragraph{Gap as a difference of pMIs.}
Under Assumption~\ref{asm:calibration}, the feedback-conditioned reward admits the analogous form $r_t^{z}(\hat{y}_t; h_t) = \operatorname{pmi}_\pi(\hat{Y}_t = \hat{y}_t;\, Z = z \mid h_t)$. Subtracting \eqref{eq:rcf-pmi} gives
\begin{equation}
  r_t^{z}(\hat{y}_t;\, h_t) - r_t^{\mathrm{cf}}(\hat{y}_t;\, h_t)
  \;=\;
  \operatorname{pmi}_\pi(\hat{Y}_t = \hat{y}_t;\, Z = z \mid h_t) - \operatorname{pmi}_\pi(\hat{Y}_t = \hat{y}_t;\, O = 1 \mid h_t).
  \label{eq:gap-pmi}
\end{equation}
The gap is the difference between two information quantities: how informative the realized choice $\hat{Y}_t = \hat{y}_t$ is about feedback $Z$ versus about outcome $O$. When these coincide the gap vanishes pointwise; otherwise it is directly characterized by their mismatch.

\paragraph{Outcome-sufficient feedback.}
\begin{assumption}[Outcome-sufficient feedback]
  \label{asm:osf}
  Feedback $Z$ is \emph{outcome-sufficient} at $h_t$ if $Z \perp\!\!\!\perp \hat{Y}_t \mid (O, h_t)$: given $(h_t, O)$, the distribution of $Z$ does not further depend on which action was taken.
\end{assumption}
We note two separate requirements for Prop.~\ref{prop:rank-preserve} to apply. First, Assumption~\ref{asm:osf} itself holds whenever $Z$ is a function of $(O, h_t)$---by construction if $Z = f(O, h_t)$, approximately if $Z$ summarizes $O$ with action-independent noise. Second, for a \emph{particular} observed value $z$, the theorem additionally requires the positive-informativeness condition $q_1(z, h_t) > q_0(z, h_t)$; these are distinct---Assumption~\ref{asm:osf} can hold while the theorem's conclusion is empty because $q_1 = q_0$ for the observed $z$. Two failure modes illustrate how the pairing can break. $Z$ may be a \emph{deterministic function of $h_t$ alone} (e.g., the gold letter read off a test prompt): Assumption~\ref{asm:osf} then holds vacuously, but the observed $z$ has $q_1(z, h_t) = q_0(z, h_t)$ so Prop.~\ref{prop:rank-preserve} does not apply. Alternatively, $Z$ may \emph{carry trajectory detail informative about $\hat{Y}_t$ beyond what $O$ conveys}, in which case Assumption~\ref{asm:osf} fails outright. The first failure mode is the one our MCQ diagnostic falls into. Writing $p(\hat{y}_t; h_t) := P_\pi(O = 1 \mid \hat{Y}_t = \hat{y}_t, h_t)$ and $q_b(z, h_t) := P(Z = z \mid O = b, h_t)$ for $b \in \{0, 1\}$, total probability under Assumption~\ref{asm:osf} is affine in $p$:
\begin{equation}
  P(Z = z \mid \hat{Y}_t = \hat{y}_t, h_t) \;=\; q_0(z, h_t) + \bigl(q_1(z, h_t) - q_0(z, h_t)\bigr)\, p(\hat{y}_t; h_t).
  \label{eq:affine}
\end{equation}
We say $z$ is \emph{positively informative} for success if $q_1(z, h_t) > q_0(z, h_t)$, i.e., observing $Z = z$ upweights success over failure.

\begin{proposition}[Rank preservation]
  \label{prop:rank-preserve}
  Under Assumptions~\ref{asm:calibration} and \ref{asm:osf}, identification conditions (i)--(iii), and positive informativeness of $z$, for any two candidates $\hat{y}_t, \hat{y}'_t$ at the same prefix,
  \begin{equation*}
    r_t^{\mathrm{cf}}(\hat{y}_t;\, h_t) \,\geq\, r_t^{\mathrm{cf}}(\hat{y}'_t;\, h_t) \quad \Longleftrightarrow \quad r_t^{z}(\hat{y}_t;\, h_t) \,\geq\, r_t^{z}(\hat{y}'_t;\, h_t).
  \end{equation*}
\end{proposition}
\begin{proof}[Proof sketch]
  By \eqref{eq:rcf-pmi}, $r_t^{\mathrm{cf}}(\hat{y}_t) = \log p(\hat{y}_t; h_t) - \log P_\pi(O{=}1 \mid h_t)$, strictly monotone increasing in $p$. By \eqref{eq:affine}, $r_t^{z}(\hat{y}_t) = \log\bigl(q_0 + (q_1 - q_0)\, p(\hat{y}_t; h_t)\bigr) - \log P_\pi(Z{=}z \mid h_t)$, strictly monotone increasing in $p$ whenever $q_1 > q_0$. Monotone transforms of the same underlying argument induce the same ordering on any finite candidate set.
\end{proof}

\begin{corollary}[One-sided witness]
  \label{cor:one-sided}
  If, additionally, $q_0(z, h_t) = 0$ (feedback $z$ is observed only under success), then $r_t^{z}(\hat{y}_t; h_t) = r_t^{\mathrm{cf}}(\hat{y}_t; h_t)$ identically: the feedback-conditioned reward equals the ideal interventional credit.
\end{corollary}
\begin{proof}
  With $q_0 = 0$, \eqref{eq:affine} reduces to $P_\pi(Z{=}z \mid \hat{Y}_t = \hat{y}_t, h_t) = q_1(z, h_t)\, p(\hat{y}_t; h_t)$. Marginalizing over $\hat{Y}_t$ gives $P_\pi(Z{=}z \mid h_t) = q_1(z, h_t)\, P_\pi(O{=}1 \mid h_t)$. Substituting both into the log-ratio definition $r_t^{z}(\hat{y}_t) = \log P_\pi(Z{=}z \mid \hat{Y}_t = \hat{y}_t, h_t) - \log P_\pi(Z{=}z \mid h_t)$ cancels $\log q_1(z, h_t)$ and leaves $\log p(\hat{y}_t; h_t) - \log P_\pi(O{=}1 \mid h_t)$, which is exactly $r_t^{\mathrm{cf}}(\hat{y}_t)$.
\end{proof}

\paragraph{Binary verifier feedback: a theoretically tight special case.}
A hypothetical tight case of Prop.~\ref{prop:rank-preserve}--Cor.~\ref{cor:one-sided} arises when the feedback variable is \emph{binary verifier feedback}: $Z = O \in \{0, 1\}$, with standard sign conventions making the observed feedback $z = 1$ only when $O = 1$. In that case $q_0(z{=}1, h_t) = P(Z{=}1 \mid O{=}0, h_t) = 0$ by construction, so Corollary~\ref{cor:one-sided} applies directly and, for every $h_t$,
\begin{equation*}
  r_t^{z}(\hat{y}_t;\, h_t) \;=\; r_t^{\mathrm{cf}}(\hat{y}_t;\, h_t),
\end{equation*}
i.e., the feedback-conditioned reward coincides identically with the ideal interventional credit. We flag this case only to characterize the boundary of Prop.~\ref{prop:rank-preserve}; \emph{none} of the feedback specifications in our main experiments---test input-output pairs for LiveCodeBench, expected tool calls for ToolAlpaca, and ground-truth answers for SciKnowEval (\S\ref{sec:setup})---take this exact binary-verifier form, so this tight coincidence does not apply to them. The empirical claims of the paper rest on the input-conditional surrogate view of \S\ref{sec:info-theory}, not on this special case; tasks with softer or noisier correctness signals (reward-model scalars, partial-credit rubrics) retain only the rank-preservation statement of Prop.~\ref{prop:rank-preserve}, with a non-trivial $q_0 > 0$.

\paragraph{Scope and limitations.}
The claims above rest on sequential ignorability as a modeling assumption for autoregressive LLMs---$h_t$ contains all observed state, but unobserved reasoning in hidden activations could in principle confound $\hat{Y}_t$ and $O$---and on Assumption~\ref{asm:osf}, which fits outcome-coded feedback (MCQ letter correctness, tool-sequence match, code-test pass) but weakens when $z$ carries trajectory-specific detail beyond $O$. Within these conditions, we establish a local rank-preservation statement; uniform bounded gap, globally optimal policy preservation, and $L_\infty$-tightness of \credit{} over raw SD remain out of scope. The main theory stays informational; this appendix is a compatible interventional reading. A rigorous empirical test of Prop.~\ref{prop:rank-preserve} would require $z$ as a noisy function of a final outcome (e.g., binary verifier correctness on a code or tool-use task), which we leave to future work.

%% ============================================================
%% G. ADDITIONAL TOKEN-LEVEL REWARD VISUALIZATIONS
%% ============================================================
\section{Additional Token-Level Reward Visualizations}
\label{app:visualizations}

We provide additional examples complementing Figure~\ref{fig:token-reward} in the main text. Each figure shows, for a single response to a LiveCodeBench coding problem: (a)~the problem statement, (b)~the self-distillation reward $\Delta V_t(x)$, and (c)~the input-specific advantage $S_t(x)$ after \credit's contrastive debiasing. Colors indicate the sign of the per-token component: \textcolor{red}{red} = positive (reinforce), \textcolor{blue}{blue} = negative (suppress); since $S_t(x)$ is one additive component of the full advantage, the color reflects the sign of $S_t(x)$ rather than the final advantage. All values are group-normalized.

\begin{figure}[h]
  \centering
  \includegraphics[width=\textwidth]{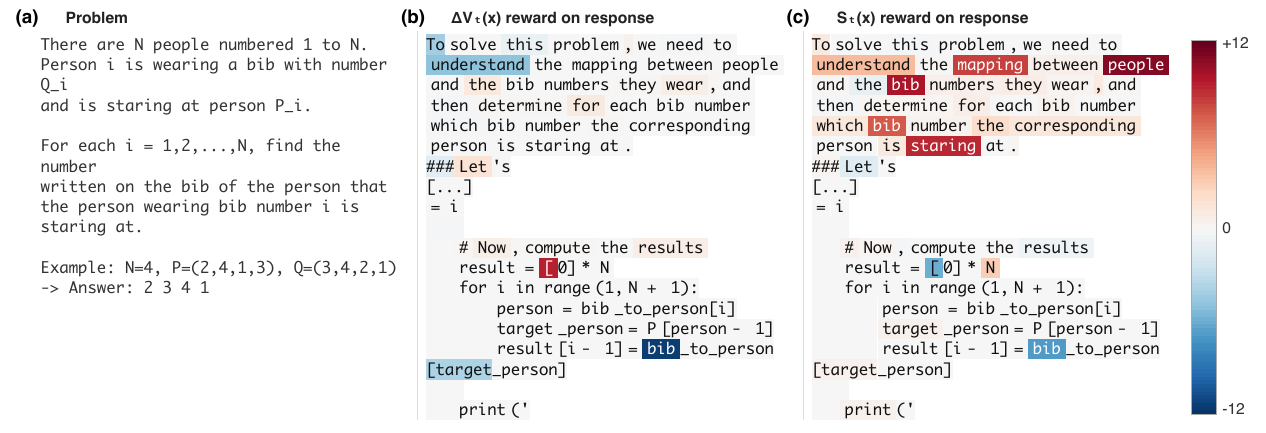}
  \caption{Index-following problem (output $Q[P[i]]$ for each $i$, given arrays $P$ and $Q$). $\Delta V_t$ is broadly positive; $S_t$ concentrates on problem-specific entities (\texttt{people}, \texttt{bib}, \texttt{staring}, \texttt{mapping}) and the indirection vocabulary the response invokes, while generic boilerplate is suppressed.}
  \label{fig:app-permutation}
\end{figure}

\begin{figure}[h]
  \centering
  \includegraphics[width=\textwidth]{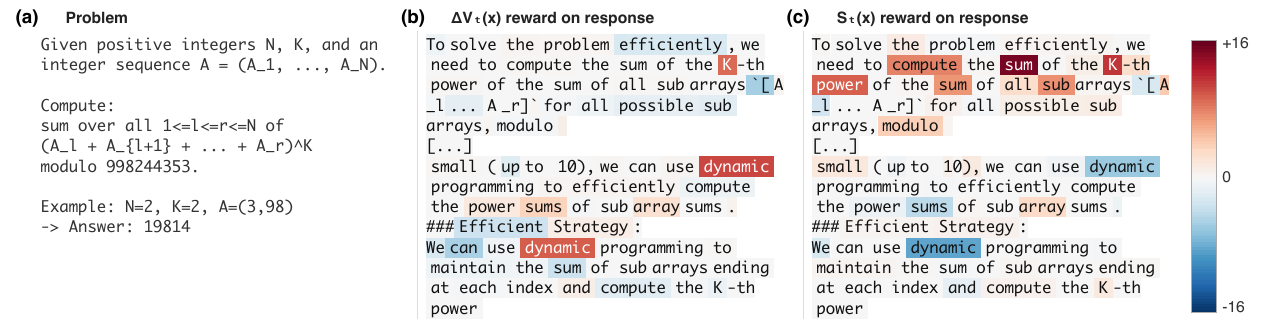}
  \caption{Sum of $K$-th powers of all subarray sums modulo a prime. $S_t$ reinforces problem-specific vocabulary (\texttt{sum}, \texttt{K}, \texttt{power}, \texttt{modulo}, \texttt{-th}) and suppresses algorithmic templates the model attempted but that do not fit the problem (\texttt{dynamic} programming, \texttt{sliding} window).}
  \label{fig:app-subarray-kpower}
\end{figure}

\begin{figure}[h]
  \centering
  \includegraphics[width=\textwidth]{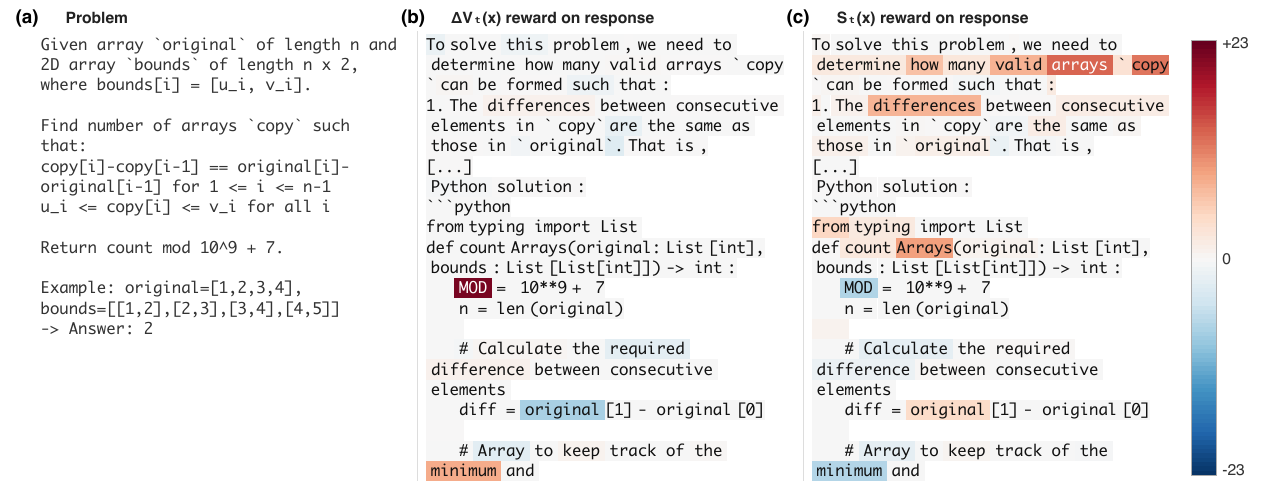}
  \caption{Counting arrays \texttt{copy} whose consecutive differences match a given array and lie within per-position bounds. $S_t$ concentrates on the problem-specific entities (\texttt{copy}, \texttt{original}, \texttt{differences}, \texttt{arrays}) and the counting framing (\texttt{how many}, \texttt{valid}, \texttt{count}, \texttt{determine}); tokens belonging to misframings (\texttt{minimum}, \texttt{MOD} used out of context) are suppressed.}
  \label{fig:app-array-original}
\end{figure}

\begin{figure}[h]
  \centering
  \includegraphics[width=\textwidth]{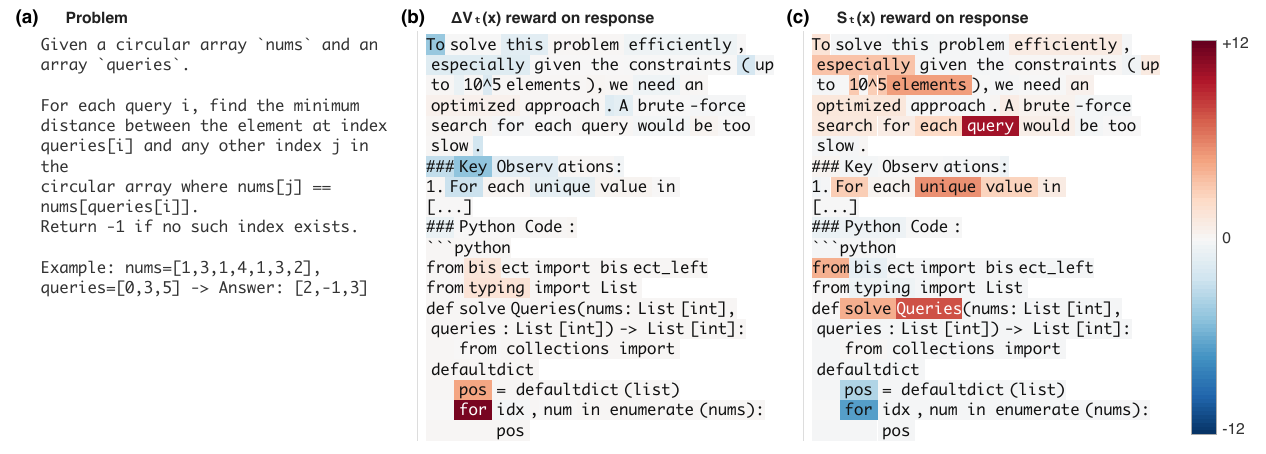}
  \caption{Per-query minimum-distance lookup on a circular array. $\Delta V_t$ is near-uniform; $S_t$ selectively reinforces tokens about the per-query data structure (\texttt{query}, \texttt{Queries}, \texttt{unique}, \texttt{elements}, \texttt{value}) needed for the lookup, while generic control-flow tokens (\texttt{for}, \texttt{pos}) carrying little input-specific information are suppressed.}
  \label{fig:app-circular}
\end{figure}

\begin{figure}[h]
  \centering
  \includegraphics[width=\textwidth]{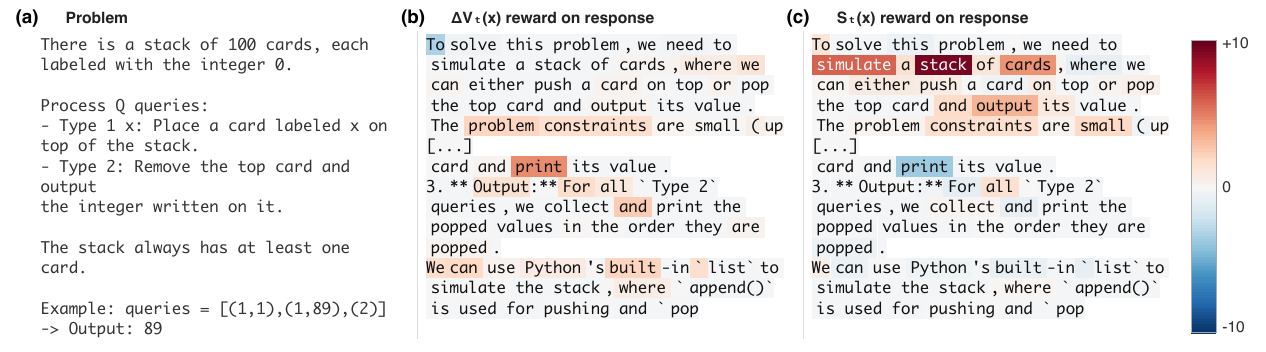}
  \caption{Stack simulation with push and pop queries on 100 cards. $S_t$ concentrates on problem-specific vocabulary (\texttt{stack}, \texttt{simulate}, \texttt{cards}, \texttt{push}) and the operation framing, while tokens unrelated to the stack abstraction (\texttt{print}, \texttt{built}, \texttt{where}) are suppressed.}
  \label{fig:app-stack}
\end{figure}

\begin{figure}[h]
  \centering
  \includegraphics[width=\textwidth]{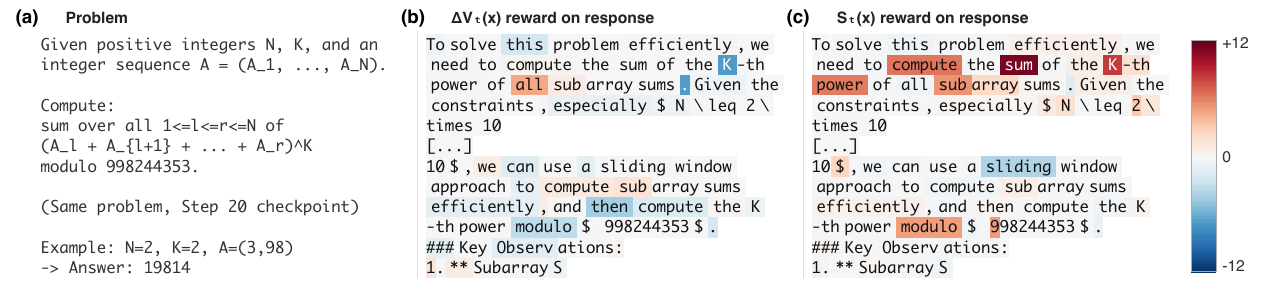}
  \caption{Same problem as Figure~\ref{fig:app-subarray-kpower}, evaluated at training step~20 (earlier checkpoint). The input-specific signal $S_t$ is weaker and less concentrated than at the later checkpoint, suggesting that the model's input-specific signal sharpens during training.}
  \label{fig:app-subarray-step20}
\end{figure}

\begin{figure}[h]
  \centering
  \includegraphics[width=\textwidth]{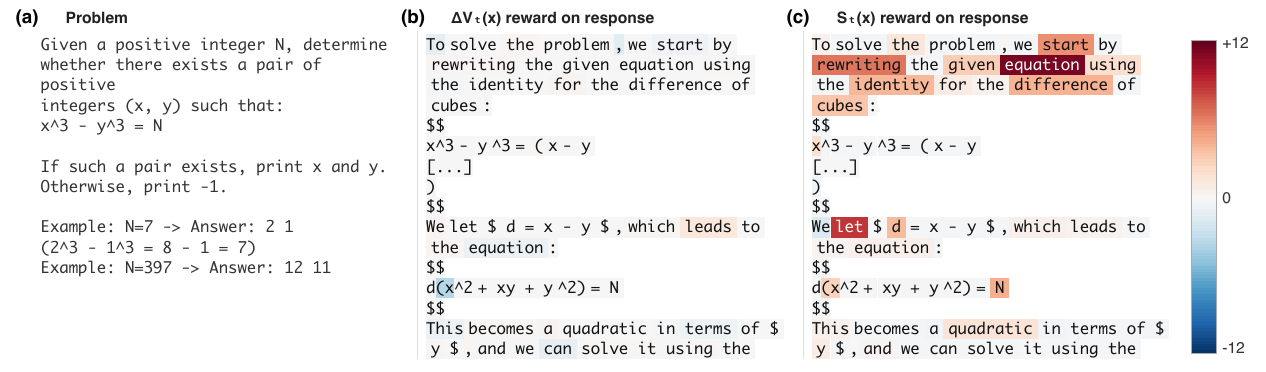}
  \caption{Diophantine problem $x^3 - y^3 = N$ over positive integers (algebraic, not geometric). $S_t$ reinforces the algebraic-manipulation tokens (\texttt{equation}, \texttt{rewriting}, \texttt{identity}, \texttt{difference}, \texttt{cubes}, \texttt{quadratic}) the response uses to apply the difference-of-cubes factorization; generic discourse tokens are suppressed.}
  \label{fig:app-cubes-step50}
\end{figure}
\FloatBarrier

%%%%%%%%%%%%%%%%%%%%%%%%%%%%%%%%%%%%%%%%%%%%%%%%%%%%%%%%%%%%

% \newpage
% \input{checklist.tex}  % NeurIPS checklist; not included in arXiv preprint

\end{document}